\documentclass{article}
\usepackage{arxiv}
\usepackage{natbib}

\usepackage[utf8]{inputenc} % allow utf-8 input
\usepackage[T1]{fontenc}    % use 8-bit T1 fonts
\usepackage{hyperref}       % hyperlinks
\usepackage{url}            % simple URL typesetting
\usepackage{booktabs}       % professional-quality tables
\usepackage{amsfonts}       % blackboard math symbols
\usepackage{nicefrac}       % compact symbols for 1/2, etc.
\usepackage{microtype}      % microtypography
\usepackage{xcolor}         % colors

\usepackage{amsmath}
\usepackage{amssymb}
\usepackage{mathtools}
\usepackage{mathrsfs}
\usepackage{amsthm}
\usepackage{cleveref}
\newtheorem{theorem}{Theorem}

\newtheorem{lemma}{Lemma}
\theoremstyle{definition}
\newtheorem{problem}{Problem}

\usepackage{booktabs}       % professional-quality tables
\usepackage{array}
\usepackage{cite}

\usepackage{algorithm,algorithmicx,algpseudocode}

\usepackage{tikz}
\usepackage{bm}

\usepackage{caption}
\usepackage{subcaption}

\usepackage{multirow}
\newcommand{\set}[1]{\{ #1 \}}
\newcommand{\ind}[1]{\mathbb{I}\left[ #1 \right]}
\newcommand{\supp}[1]{\operatorname{supp}( #1 )}
\newcommand{\lsupp}[1]{\operatorname{lsupp}( #1 )}
\newcommand{\sgn}[1]{\operatorname{sgn}\left( #1 \right)}
\DeclareMathOperator*{\argmin}{\operatorname{arg\, min}}

\renewcommand{\epsilon}{\varepsilon}
\crefname{algorithm}{Algorithm}{Algorithms}
\crefname{table}{Table}{Tables}
\crefname{figure}{Figure}{Figures}
\crefname{section}{Section}{Sections}
\crefname{problem}{Problem}{Problems}
\crefname{theorem}{Theorem}{Theorems}
\crefname{proposition}{Proposition}{Propositions}
\crefname{lemma}{Lemma}{Lemmas}
\crefname{remark}{Remark}{Remarks}

\title{Learning Locally Interpretable Rule Ensemble}

\author{%
  Kentaro Kanamori \\
  Artificial Intelligence Laboratory, Fujitsu Ltd. \\
  \texttt{k.kanamori@fujitsu.com} \\
}

\begin{document}

\maketitle              % typeset the header of the contribution
\begin{abstract}
This paper proposes a new framework for learning a rule ensemble model that is both accurate and interpretable\footnote{To appear in the 2023 European Conference on Machine Learning and Principles and Practice of Knowledge Discovery in Databases (ECMLPKDD 2023). }. A rule ensemble is an interpretable model based on the linear combination of weighted rules. In practice, we often face the trade-off between the accuracy and interpretability of rule ensembles. That is, a rule ensemble needs to include a sufficiently large number of weighted rules to maintain its accuracy, which harms its interpretability for human users. To avoid this trade-off and learn an interpretable rule ensemble without degrading accuracy, we introduce a new concept of interpretability, named local interpretability, which is evaluated by the total number of rules necessary to express individual predictions made by the model, rather than to express the model itself. Then, we propose a regularizer that promotes local interpretability and develop an efficient algorithm for learning a rule ensemble with the proposed regularizer by coordinate descent with local search. Experimental results demonstrated that our method learns rule ensembles that can explain individual predictions with fewer rules than the existing methods, including RuleFit, while maintaining comparable accuracy. 
\end{abstract}

\section{Introduction}
In the applications of machine learning models to high-stake decision-making such as loan approvals, \emph{interpretability} has been recognized as an important element~\citep{Rudin:NMI2019,Rudin:SS2022,Miller:AI2019}. 
One of the well-known interpretable models is a \emph{rule model}, including decision trees~\citep{Hu:NIPS2019}, rule lists~\citep{Angelino:KDD2017}, and rule sets~\citep{Lakkaraju:KDD2016}. 
Because rule models are expressed using logical rules that are easy to understand, they can explain how they make predictions in an interpretable manner by themselves. 
Such explanation helps human users ensure transparency for their critical decision-making, as well as discover new knowledge from data~\citep{Freitas:EN2014}. 
% In this study, we focus on a \emph{rule ensemble}~\citep{Friedman:AAS2008,Wei:ICML2019}, which is a rule model consisting of a set of logical rules and their corresponding weight values. 
% For a given input, a rule ensemble makes a prediction based on the linear combination of the weighted rules that the input satisfies. 
In this study, we focus on a \emph{rule ensemble}~\citep{Friedman:AAS2008,Wei:ICML2019}, which is a rule model based on the linear combination of weighted rules. 
For a given input, a rule ensemble makes a prediction depending on the sum of the weights corresponding to the rules that the input satisfies. 

One of the main obstacles to learning rule ensembles is the trade-off between accuracy and interpretability. 
While the interpretability of rule models has several definitions depending on their forms and applications~\citep{Lipton:Queue2018,Doshi-Velez:arxiv2017}, a common criterion for evaluating the interpretability of a model is the total number of rules required to express the model~\citep{Lage:HCOMP2019,Freitas:EN2014}. 
Due to the cognitive limitations of human users, a model should consist of as few rules as possible, even if the model belongs to the class of inherently interpretable models~\citep{Rudin:NMI2019,Lakkaraju:KDD2016}. 
In practice, however, a rule ensemble requires a sufficiently large number of weighted rules to maintain generalization performance~\citep{Nalenz:AISTATS2022}. 
Therefore, we often need to compromise interpretability to maintain accuracy when learning rule ensembles. 

In general, a major approach to addressing the accuracy-interpretability trade-off is to make each individual prediction, rather than a model itself, interpretable. 
For example, if we need to validate an undesired prediction result (e.g., high risk of default) made by a model, it is often sufficient to explain a reason why the model outputs the prediction in an interpretable way, even if the model consists of too many rules to interpret~\citep{Caruana:KDD2015}. 
To extract an explanation for each individual prediction from a learned model, several model-agnostic methods, such as LIME and SHAP, have been proposed~\citep{Ribeiro:KDD2016,Lundberg:NIPS2017}. 
However, because most of these methods construct explanations by locally approximating a model, recent studies have pointed out the risk that their explanations are inconsistent with the actual behavior of the model~\citep{Rudin:NMI2019,Jacovi:ACL2020,Yoon:TMLR2022}. 
To avoid this risk and provide faithful explanations, we need to explain each individual prediction with the rules that the model actually uses to make the prediction, without approximation. 

In this paper, we propose \emph{locally interpretable rule ensemble (LIRE)}, a new framework for learning accurate and interpretable rule ensembles. 
While a number of weighted rules are required to maintain the accuracy of a rule ensemble model, not all of them are required to make each individual prediction by the model. 
More precisely, only the weighted rules that a given input satisfies are required to express its prediction, and the other rules are not by the definition of rule ensembles. 
This fact suggests a chance to learn a rule ensemble with a sufficient number of weighted rules to maintain accuracy but that can express individual predictions using a few weighted rules~\citep{Rudin:NMI2019}. 
Motivated by this fact, we aim to learn a rule ensemble that can explain individual predictions with as few weighted rules as possible, which we refer to as \emph{local interpretability}. 
To this end, we introduce a regularizer that promotes local interpretability, and propose an efficient algorithm for learning a rule ensemble with the proposed regularizer.

\begin{table}[p]
    \centering
    % \small
    \caption{
    Examples of an input vector $\bm{x}$ and rule ensemble classifiers on the Adult dataset. 
    The classifiers predict an input as ``Income $<$ \$50K" if the sum of the weights of the satisfied rules is greater than their intercept. 
    In (b) and (c), rules that the input $\bm{x}$ satisfies are highlighted in boldface. 
    }
    \begin{subtable}[h]{\textwidth}
        \centering
        \caption{Input vector $\bm{x}$ with the label ``Income $<$ \$50K"}
        \begin{tabular}{lc}
        \toprule
        \textbf{Feature} & \textbf{Value} \\
        \midrule
            {Age} & $39$ \\
            {fnlwgt} & $120985$ \\
            {Education-Num} & $9$ \\
            {Capital-Gain} & $0$ \\
            {Capital-Loss} & $0$ \\
            {House-per-week} & $40$ \\
            {Workclass} & Private \\
            {Education} & HS-grad \\
            {Marital-Status} & Divorced \\
            {Occupation} & Other-service \\
            {Relationship} & Own-child \\
            {Race} & White \\
            {Sex} & Male \\
            {Country} & United-States  \\
        \bottomrule
        \end{tabular}
        \label{tab:intro:example:input}
    \end{subtable}
    \hfill
    % \subfloat[RuleFit~\citep{Friedman:AAS2008} (test AUC: 0.859)]{
    % \subfloat[RuleFit~\citep{Friedman:AAS2008} (test F1 score: 0.894)]{
    \begin{subtable}[h]{\textwidth}
        \centering
        \caption{RuleFit~\citep{Friedman:AAS2008} (intercept: $0.0$, test accuracy: $83.0\%$)}
        \begin{tabular}{lc}
        \toprule
            \textbf{Rule} & \textbf{Weight} \\
        \midrule
            \textbf{Education-Num $\leq$ 12 \& Capital-Gain $\leq$ 5119} & $\bm{1.006}$ \\
            \textbf{Marital-Status $\not=$ Married-civ-spouse \& Education $\not=$ Prof-school}	& $\bm{0.644}$ \\
            \textbf{Capital-Loss $\leq$ 1820 \& Marital-Status $\not=$ Married-civ-spouse} & $\bm{0.411}$ \\
            \textbf{Marital-Status $\not=$ Married-civ-spouse \& Hours-per-week $\leq$ 44} & $\bm{0.312}$ \\
            \textbf{Age $>$ 31 \& Sex $=$ Male} & $\bm{-0.191}$ \\
            \textbf{Marital-Status $\not=$ Married-civ-spouse \& Education $\not=$ Masters} & $\bm{0.050}$ \\
            Marital-Status $=$ Married-civ-spouse \& Education $\not=$ HS-grad & $-0.027$ \\
            Hours-per-week $>$ 43 \& Marital-Status $\not=$ Never-married & $-0.014$ \\
        \bottomrule
        \end{tabular}
        \label{tab:intro:example:rulefit}
    \end{subtable}
    \hfill
    % \subfloat[LIRE (ours, test AUC: 0.866)]{
    % \subfloat[LIRE (ours, test F1 score: 0.900)]{
    \begin{subtable}[h]{\textwidth}
        \centering
        \caption{LIRE (ours, intercept: $-1.637$, test accuracy: $84.2\%$)}
        \begin{tabular}{lc}
        \toprule
            \textbf{Rule} & \textbf{Weight} \\
        \midrule
            Capital-Gain $>$ 5119 & $-1.536$ \\
            \textbf{Relationship $=$ Own-child \& Hours-per-week $\leq$ 49} & $\bm{1.255}$ \\
            Capital-Loss $>$ 1820 \& Capital-Loss $\leq$ 1978 & $-1.245$ \\
            Marital-Status $=$ Married-civ-spouse & $-1.192$ \\
            \textbf{Hours-per-week $\leq$ 43 \& Occupation $=$ Other-service} & $\bm{0.906}$ \\
    	Education-Num $>$ 12 & $-0.801$ \\
            Relationship $\not=$ Own-child \& Capital-Gain $>$ 5095 & $-0.661$ \\
         \bottomrule
        \end{tabular}
        \label{tab:intro:example:lire}
    \end{subtable}
    \label{tab:intro:example}
\end{table}

\subsubsection*{Our Contributions}
Our contributions are summarized as follows:
\begin{itemize}
    \item 
    We introduce a new concept for evaluating the interpretability of rule ensembles. 
    Our concept, named local interpretability, is evaluated by the total number of weighted rules that are necessary to express each individual prediction locally, rather than to express the entire model globally. 
    \item 
    We propose a regularizer that promotes the local interpretability of a rule ensemble, and formulate a task of learning a locally interpretable rule ensemble (LIRE) classifier. 
    Then, we propose an efficient algorithm for learning a LIRE classifier by coordinate descent with local search. 
    \item 
    We conducted experiments on real datasets to evaluate the efficacy of LIRE. 
    We confirmed that our method can learn rule ensembles that are more locally interpretable than the existing methods such as RuleFit~\citep{Friedman:AAS2008}, while maintaining accuracy and entire interpretability comparable to them. 
\end{itemize}

\cref{tab:intro:example} presents a demonstration of our framework on the Adult dataset~\citep{Dua:2019}. 
While \cref{tab:intro:example}a shows an example of an input vector $\bm{x}$, Tables~\ref{tab:intro:example}b and \ref{tab:intro:example}c present examples of rule ensemble classifiers leaned by RuleFit~\citep{Friedman:AAS2008} and our LIRE. 
In Tables~\ref{tab:intro:example}b and \ref{tab:intro:example}c, we denote the weighted rules that the input $\bm{x}$ satisfies in boldface, and the average total number of them was $\bm{3.8}$ for RuleFit and $\bm{1.1}$ for LIRE, respectively. 
\cref{tab:intro:example} demonstrates that our LIRE 
(i)~could make an accurate prediction for $\bm{x}$ with fewer weighted rules than RuleFit, and
(ii)~achieved the test accuracy comparable to RuleFit. 
These results suggest that our method can learn a locally interpretable rule ensemble without degrading accuracy. 
% This result suggests that our method learns a rule ensemble that can explain individual predictions with a few rules %without degrading prediction accuracy. 

\subsubsection*{Notation}
For a positive integer $n \in \mathbb{N}$, we write $[n] \coloneqq \set{1,\dots,n}$.
For a proposition $\psi$, $\ind{\psi}$ denotes the indicator of $\psi$; that is, $\ind{\psi}=1$ if $\psi$ is true, and $\ind{\psi}=0$ if $\psi$ is false. 
Throughout this paper, we consider a \emph{binary classification problem} as a prediction task. 
Note that our framework introduced later can also be applied to regression problems. 
We denote input and output domains $\mathcal{X} \subseteq \mathbb{R}^{D}$ and $\mathcal{Y} = \set{-1, +1}$, respectively. 
Let a tuple $(\bm{x},y)$ of an input vector $\bm{x} \in \mathcal{X}$ and output label $y \in \mathcal{Y}$ be an \emph{example}, and the set $S = \set{(\bm{x}_n, y_n)}_{n=1}^{N}$ be a \emph{sample} with $N$ examples. 
We call a function $h \colon \mathcal{X} \to \mathcal{Y}$ a \emph{classifier}. 
Let $l \colon \mathcal{Y} \times \mathbb{R} \to \mathbb{R}_{\geq 0}$ be a loss function, such as the logistic loss, hinge loss, or exponential loss~\citep{Mohri:2012:Foundations}. 
% We assume that a loss function $l$ is convex. 

\section{Rule Ensemble}
A \emph{rule ensemble} is a model consisting of a set of rules and their corresponding weights~\citep{Friedman:AAS2008}. 
Each rule is expressed as a form of a conjunction of features (e.g., ``Age $>$ 31 \& Sex $=$ Male" as shown in \cref{tab:intro:example}), and has a corresponding weight value. 
Given an input $\bm{x} \in \mathcal{X}$, a rule ensemble makes a prediction depending on the linear combination of the weighted rules that the input satisfies. 
For binary classification, a rule ensemble classifier $h_{\bm{\alpha}} \colon \mathcal{X} \to \mathcal{Y}$ is defined as
\begin{align*}
    h_{\bm{\alpha}}(\bm{x}) \coloneqq \sgn{\sum_{m=1}^{M} \alpha_m \cdot r_m(\bm{x})},
\end{align*}
where $r_m \colon \mathcal{X} \to \set{0,1}$ is a rule, $\alpha_m \in \mathbb{R}$ is a weight corresponding to $r_m$, and $M \in \mathbb{N}$ is the total number of rules. 
We denote the \emph{decision function} $f_{\bm{\alpha}} \colon \mathcal{X} \to \mathbb{R}$ of $h_{\bm{\alpha}}$ by $f_{\bm{\alpha}}(\bm{x}) \coloneqq \sum_{m=1}^{M} \alpha_m \cdot r_m(\bm{x})$, i.e., $h_{\bm{\alpha}}(\bm{x}) = \sgn{f_{\bm{\alpha}}(\bm{x})}$. 

To learn a rule ensemble $h_{\bm{\alpha}}$ from a given sample $S$, we first need to obtain a set of rules $R = \set{r_1, \dots, r_M}$ from $S$. 
However, it is computationally difficult to enumerate all the candidate rules on $\mathcal{X} \subseteq \mathbb{R}^D$ because their size grows exponentially with $D$~\citep{Wei:ICML2019,Nakagawa:KDD2016,Kato:PAMI2023}. 
To avoid enumerating all of them, we need to efficiently generate a subset of candidate rules that can improve the accuracy of $h_{\bm{\alpha}}$. 

Another challenge is to learn a sparse weight vector $\bm{\alpha} = (\alpha_1, \dots, \alpha_M) \in \mathbb{R}^M$. 
By definition, a rule $r_m$ with $\alpha_m = 0$ does not contribute to the predictions of a rule ensemble $h_{\bm{\alpha}}$. 
% Hence, reducing the rules $r_m$ with $\alpha_m \not= 0$ as much as possible is desirable to ensure interpretability. 
% To ensure interpretability, reducing the rules $r_m$ with $\alpha_m \not= 0$ is desirable. 
% However, a rule ensemble often requires a sufficiently large number of rules with non-zero weights to achieve good generalization~\citep{Nalenz:AISTATS2022}. 
While reducing the total number of the rules $r_m$ with non-zero weights is essential to ensure the interpretability of $h_{\bm{\alpha}}$, it often harms the generalization performance of $h_{\bm{\alpha}}$~\citep{Nalenz:AISTATS2022}. 
Therefore, we need to find a weight vector $\bm{\alpha}$ that is as sparse as possible while maintaining the accuracy of $h_{\bm{\alpha}}$. 

\begin{figure}[t]
    \begin{subfigure}[h]{0.45\textwidth}
        \centering
        \includegraphics[width=\linewidth]{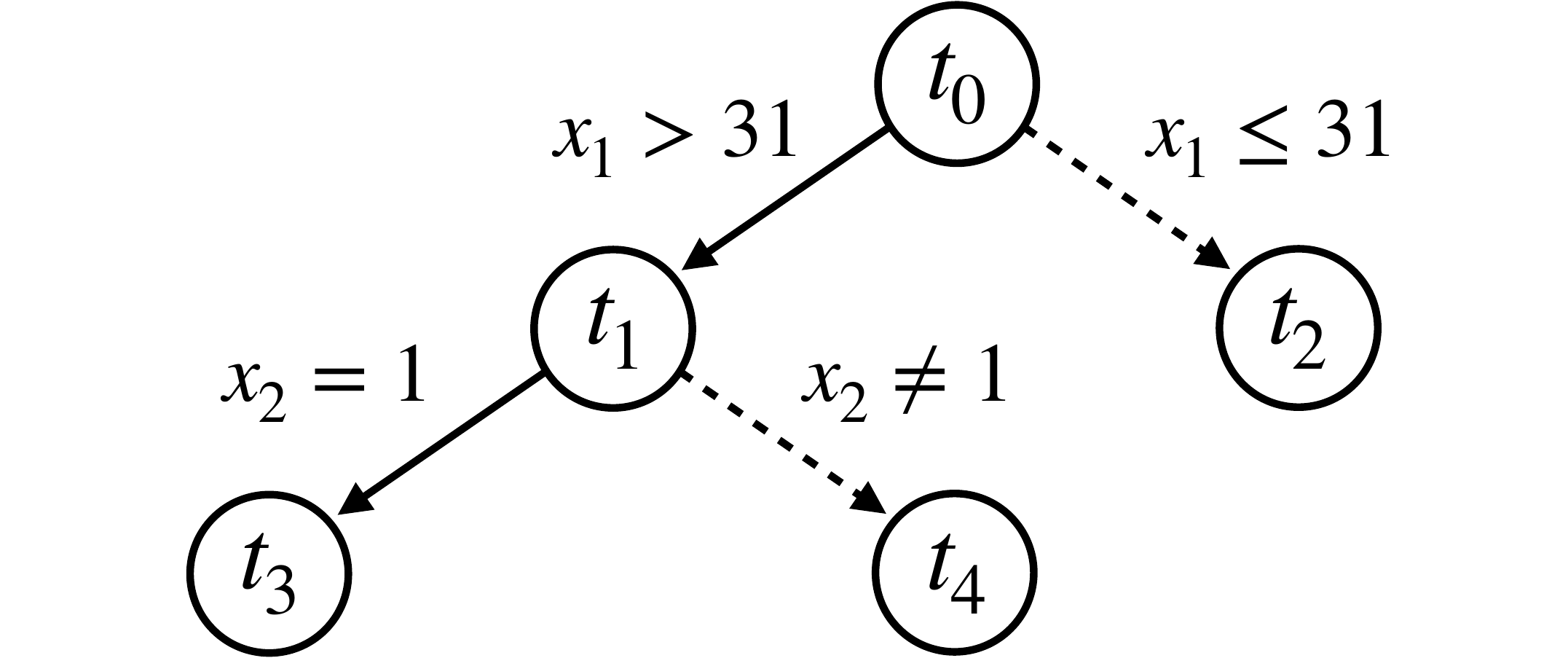}
        \caption{Decision tree}
        \label{tab:re:rule:dt}
    \end{subfigure}
    \hfill
    \begin{subfigure}[h]{0.55\textwidth}
        \centering
        \begin{tabular}{cl}
        \toprule
            \;\;\;\textbf{Node}\;\;\; & \textbf{Rule} \\
        \midrule
            $t_1$ & $r_1(\bm{x}) = \ind{x_1 > 31}$ \\
            $t_2$ & $r_2(\bm{x}) = \ind{x_1 \leq 31}$ \\
            $t_3$ & $r_3(\bm{x}) = \ind{x_1 > 31} \cdot \ind{x_2 = 1}$ \\
            $t_4$ & $r_4(\bm{x}) = \ind{x_1 > 31} \cdot \ind{x_2 \not= 1}$ \\
        \bottomrule
        \end{tabular}
        \caption{Decomposed rules}
        \label{tab:re:rule:decompose}
     \end{subfigure}
    \caption{Examples of a decision tree and its decomposed rules $R = \set{r_1, r_2, r_3, r_4}$. }
    \label{fig:re:rule}
\end{figure}

\subsection{RuleFit}
A popular practical framework for learning rule ensembles is \emph{RuleFit} proposed by \citet{Friedman:AAS2008}. 
RuleFit consists of two steps; it first extracts a set of candidate rules $R$ from a learned ensemble of decision trees, and then optimizes a sparse weight vector $\bm{\alpha}$ through the $\ell_1$-regularization. 

% \subsubsection{Rule extraction}
\paragraph{Rule Extraction.}
Given a sample $S$, RuleFit first learns a tree ensemble model, such as random forests~\citep{Breiman:ML2001} and gradient boosting decision trees~\citep{Ke:NIPS2017}, on $S$. 
It then decomposes each decision tree of the ensemble into a set of rules and collects the decomposed rules over the entire ensemble as $R$. 
Because there exist several fast algorithms for learning tree ensemble models, we can efficiently obtain a set of rules $R$ that can improve the accuracy of a rule ensemble $h_{\bm{\alpha}}$ on $S$. 

\cref{fig:re:rule} shows an example of a decision tree and its decomposed rules. 
By collecting the branching conditions on the path between the root and each node of the decision tree in \cref{fig:re:rule}a, we can obtain the set of rules shown in \cref{fig:re:rule}b.  
For example, we can obtain the rule $r_3$ from the node $t_3$ by combining the conditions $x_1 > 31$ and $x_2 = 1$ on the path between the root $t_0$ and $t_3$.

% \subsubsection{Weight optimization}
\paragraph{Weight Optimization.}
For a set of extracted rules $R = \set{r_1, \dots, r_M}$, we optimize a weight vector $\bm{\alpha} \in \mathbb{R}^{M}$ on the sample $S$. 
Because the size of the extracted rules $M$ grows in proportion to the total number of leaves in the tree ensemble, we need to keep $\bm{\alpha}$ as sparse as possible to avoid overfitting and ensure interpretability. 
To learn a sparse weight vector $\bm{\alpha}$, RuleFit uses the $\ell_1$-regularization (i.e., the Lasso penalty~\citep{Tibshirani:JRSS1994}). 
Specifically, RuleFit solves the following learning problem. 

\begin{problem}[RuleFit]\label{prob:rulefit}
    For a given sample $S = \set{(\bm{x}_n, y_n)}_{n=1}^{N}$,
    set of $M$ rules $R = \set{r_1, \dots, r_M}$, 
    and trade-off parameter $\gamma \geq 0$, 
    find an optimal solution $\bm{\alpha}^\ast \in \mathbb{R}^M$ to the following problem:
    \begin{align*}
        \bm{\alpha}^\ast = \argmin_{\bm{\alpha} \in \mathbb{R}^M} L(\bm{\alpha} \mid S) + \gamma \cdot \Omega_1(\bm{\alpha}),
    \end{align*}
    where $L(\bm{\alpha} \mid S) \coloneqq \frac{1}{N} \sum_{n=1}^{N} l(y_n, f_{\bm{\alpha}}(\bm{x}_n))$ is the empirical risk on $S$, and $\Omega_1(\bm{\alpha}) \coloneqq \| \bm{\alpha} \|_1$ is the $\ell_1$-regularization term that promotes the sparsity of $\bm{\alpha}$. 
\end{problem}

Note that we can efficiently solve \cref{prob:rulefit} using the existing algorithms for learning generalized additive models with the $\ell_1$-regularization when the loss function $l$ is convex~\citep{Friedman:TR2003}. 
A common choice of the loss function $l$ for a binary classification task is the logistic loss $l(y, f_{\bm{\alpha}}(\bm{x})) = \log(1 + e^{- y \cdot f_{\bm{\alpha}}(\bm{x})})$~\citep{Wei:ICML2019,Ustun:JMLR2019}.

\section{Problem Formulation}
This section presents our proposed framework, named \emph{Locally Interpretable Rule Ensemble~(LIRE)}. 
We introduce \emph{local interpretability}, a new concept of interpretability for rule ensembles that is evaluated by the total number of rules required to express individual predictions, rather than the model itself. 
Then, we propose a regularizer that promotes local interpretability, and formulate the task of learning a rule ensemble with our local interpretability regularizer.

\subsection{Local Interpretability of Rule Ensemble}
In general, the interpretability of a rule ensemble $h_{\bm{\alpha}}$ is evaluated by the total number of rules with non-zero weights~\citep{Friedman:AAS2008,Wei:ICML2019}. 
Let $\supp{\bm{\alpha}} \coloneqq \set{m \in [M] \mid \alpha_m \not= 0}$ be the set of rules with non-zero weights, which we call the \emph{support} of $h_{\bm{\alpha}}$. 
By definition, a rule ensemble classifier $h_{\bm{\alpha}}$ can be expressed using only weighted rules in the support, i.e., $h_{\bm{\alpha}}(x) = \operatorname{sgn}( \sum_{m \in \supp{\bm{\alpha}}} \alpha_m \cdot r_m(x) )$. 
To ensure interpretability, the existing methods reduce the support size $|\supp{\bm{\alpha}}|$ by the $\ell_1$-regularization~\citep{Tibshirani:JRSS1994}. 
In practice, however, since there is a trade-off between the support size of a rule ensemble and its generalization performance, we often need to compromise interpretability to maintain accuracy~\citep{Nalenz:AISTATS2022}. 

On the other hand, not all of the rules with non-zero weights are used for making the individual prediction $h_{\bm{\alpha}}(\bm{x})$ of each input $\bm{x} \in \mathcal{X}$. 
This is because a rule $r_m$ with $\alpha_m \not= 0$ but $r_m(\bm{x}) = 0$ does not contribute to the prediction result $h_{\bm{\alpha}}(\bm{x})$ by the definition of rule ensembles. 
It suggests that we only need the rules $r_m$ with $\alpha_m \not= 0$ and $r_m(\bm{x}) = 1$ to express the individual prediction $h_{\bm{\alpha}}(\bm{x})$ for a given input $\bm{x}$. 
In some practical situations (e.g., loan approvals and medical diagnoses), even if a model itself is too complex to interpret, it is often sufficient to explain its individual predictions in an interpretable manner~\citep{Rudin:NMI2019,Yang:ICML2022,Caruana:KDD2015}. 

Motivated by the above facts, we introduce a new concept of interpretability for a rule ensemble model from the perspective of its individual predictions rather than the model itself. 
We focus on learning a rule ensemble that can express individual predictions using a few rules with non-zero weights, which we call \emph{local interpretability}. 
To evaluate the local interpretability of a rule ensemble $h_{\bm{\alpha}}$, we define the \emph{local support} of $h_{\bm{\alpha}}$ for an input $\bm{x}$ by
\begin{align*}
    \lsupp{\bm{\alpha} \mid \bm{x}} \coloneqq \set{m \in [M] \mid \alpha_m \not= 0 \land r_m(\bm{x}) = 1}.
\end{align*}
By definition, the prediction $h_{\bm{\alpha}}(\bm{x})$ for $\bm{x}$ can be expressed using only weighted rules in the local support, i.e., $h_{\bm{\alpha}}(x) = \operatorname{sgn}( \sum_{m \in \lsupp{\bm{\alpha} \mid x}} \alpha_m )$. 
To ensure local interpretability, we aim to reduce the local support size $|\lsupp{\bm{\alpha} \mid \bm{x}}|$ for each $\bm{x}$ in a given sample $S$ as much as possible. 

To promote the local interpretability of a rule ensemble $h_{\bm{\alpha}}$, we propose a \emph{local interpretability regularizer}. 
By definition, $|\lsupp{\bm{\alpha} \mid \bm{x}}| \leq |\supp{\bm{\alpha}}|$ holds for any $\bm{x} \in \mathcal{X}$, which implies that reducing the support size $|\supp{\bm{\alpha}}|$ leads to a decrease in the upper bound on the local support size $|\lsupp{\bm{\alpha} \mid \bm{x}}|$. 
However, we need to avoid achieving local interpretability by reducing the support size since it may harm the accuracy of $h_{\bm{\alpha}}$. 
% However, because the support size sufficient to maintain accuracy depends on the task and application, it is desirable to avoid achieving local interpretability by reducing the support size. 
% Because the support size sufficient to maintain accuracy depends on the task, we need to avoid achieving local interpretability by reducing the support size. 
To control the local support size separately from the support size, we define our local interpretability regularizer $\Omega_\mathrm{L}$ as
\begin{align*}
    \Omega_\mathrm{L}(\bm{\alpha} \mid S) \coloneqq \frac{1}{N} \sum_{n=1}^{N} \frac{|\lsupp{\bm{\alpha} \mid \bm{x}_n}|}{|\supp{\bm{\alpha}}|}.
\end{align*}
That is, we evaluate the ratio of the local support size $|\lsupp{\bm{\alpha} \mid \bm{x}}|$ to the support size $|\supp{\bm{\alpha}}|$ for each input $\bm{x}$ in a sample $S$ and average them over $S$. 
Minimizing our regularizer $\Omega_\mathrm{L}$ allows us to reduce the average local support size $|\lsupp{\bm{\alpha} \mid \bm{x}}|$ without directly constraining the support size $|\supp{\bm{\alpha}}|$.

\subsection{Locally Interpretable Rule Ensemble}
We now formulate our problem of learning a \emph{locally interpretable rule ensemble (LIRE)} classifier. 
As with RuleFit~\citep{Friedman:AAS2008}, we assume that we have a set of rules $R$ by extracting them from a tree ensemble leaned on a given sample $S$ in advance. 
Then, we learn a weight vector $\bm{\alpha}$ that minimizes the empirical risk $L(\bm{\alpha} \mid S) = \frac{1}{N} \sum_{n=1}^{N} l(y_n, f_{\bm{\alpha}}(\bm{x}_n))$ on $S$ with the regularizers on its interpretability. 

\begin{problem}[LIRE]\label{prob:lire}
    For a given sample $S = \set{(\bm{x}_n, y_n)}_{n=1}^{N}$,
    set of $M$ rules $R = \set{r_1, \dots, r_M}$, 
    and hyper-parameters $\gamma, \lambda \geq 0$, 
    find an optimal solution $\bm{\alpha}^\ast \in \mathbb{R}^M$ to the following problem:
    \begin{align*}
        \bm{\alpha}^\ast = \argmin_{\bm{\alpha} \in \mathbb{R}^M} G_{\gamma, \lambda}(\bm{\alpha} \mid S) \coloneqq L(\bm{\alpha} \mid S) + \gamma \cdot \Omega_\mathrm{G}(\bm{\alpha}) + \lambda \cdot \Omega_\mathrm{L}(\bm{\alpha} \mid S),
    \end{align*}
    where $\Omega_\mathrm{G}(\bm{\alpha}) \coloneqq | \supp{\bm{\alpha}} |$ is the global interpretability regularizer, and $\Omega_\mathrm{L}(\bm{\alpha} \mid S) = \frac{1}{N} \sum_{n=1}^{N} \frac{|\lsupp{\bm{\alpha} \mid \bm{x}_n}|}{|\supp{\bm{\alpha}}|}$ is the local interpretability regularizer. 
\end{problem}

By solving \cref{prob:lire}, we are expected to obtain an accurate rule ensemble $h_{\bm{\alpha}^\ast}$ whose local support size $|\lsupp{\bm{\alpha} \mid \bm{x}}|$ is small on average. 
We can control the strength of the global and local interpretability regularizers by tuning the parameters $\gamma$ and $\lambda$. 
Note that our global interpretability regularizer $\Omega_\mathrm{G}(\bm{\alpha})$ is equivalent to the $\ell_0$-regularization term $\| \bm{\alpha} \|_0$, and the $\ell_1$-regularization term $\| \bm{\alpha} \|_1$ used in RuleFit can be regarded as a convex relaxation of $\| \bm{\alpha} \|_0$~\citep{Tibshirani:JRSS1994}. 
We employ $\Omega_\mathrm{G}(\bm{\alpha})$ to penalize the support size of $\bm{\alpha}$ more directly than the $\ell_1$-regularization without degrading the generalization performance of $h_{\bm{\alpha}}$~\citep{Dedieu:JMLR2021,Liu:AISTATS2022}.

\section{Optimization}
In this section, we propose a learning algorithm for a LIRE classifier. 
Because our local interpretability regularizer $\Omega_\mathrm{L}$ is neither differentiable nor convex due to its combinatorial nature, efficiently finding an exact optimal solution to \cref{prob:lire} is computationally challenging, even if the loss function $l$ and the global interpretability regularizer $\Omega_\mathrm{G}$ are differentiable and convex. 
To avoid this difficulty, we propose to extend the existing fast algorithms for learning $\ell_0$-regularized generalized additive classifiers~\citep{Dedieu:JMLR2021,Liu:AISTATS2022} to our learning problem.

\begin{algorithm}[t]
\flushleft
    % \small
    \caption{Coordinate descent algorithm with local search for learning LIRE.}
    \begin{algorithmic}[1]
        \Require{
            a sample $S$,
            a set of rules $R = \set{r_1, \dots, r_M}$ extracted from a tree ensemble learned on $S$ in advance, 
            trade-off parameters $\gamma, \lambda \geq 0$,
            an initial weight vector $\bm{\alpha}^{(0)} \in \mathbb{R}^M$,
            and a maximum number of iterations $I \in \mathbb{N}$ (e.g., $I = 5000$). 
        }
        \Ensure{a weight vector $\bm{\alpha}^{(i)}$.}
        \For{$i = 1, 2, \dots, I$}
            \State $\bm{\alpha}^{(i)} \leftarrow \bm{\alpha}^{(i-1)}$;
            \For{$m \in \supp{\bm{\alpha}^{(i)}}$}
                \State $\alpha^{(i)}_{m} \leftarrow \argmin_{\alpha_{m} \in \mathbb{R}} G_{\gamma, \lambda}(\bm{\alpha}^{(i)}_{-m} + \alpha_{m} \cdot \bm{e}_{m} \mid S)$; \Comment{$\bm{\alpha}^{(i)}_{-m} \coloneqq \bm{\alpha}^{(i)} - \alpha^{(i)}_m \cdot \bm{e}_m$} \label{line:delete}
                \If{$\alpha^{(i)}_m = 0$}
                    \State \textbf{break}; \Comment{Delete $m$}
                \EndIf
                \For{$m' \in [M] \setminus \supp{\bm{\alpha}^{(i)}}$}
                    \State $\alpha^\ast_{m'} \leftarrow \argmin_{\alpha_{m'} \in \mathbb{R}} G_{\gamma, \lambda}(\bm{\alpha}^{(i)}_{-m} + \alpha_{m'} \cdot \bm{e}_{m'} \mid S)$; \label{line:replace}
                    \If{$G_{\gamma, \lambda}(\bm{\alpha}^{(i)}_{-m} + \alpha^\ast_{m'} \cdot \bm{e}_{m'} \mid S) < G_{\gamma, \lambda}(\bm{\alpha}^{(i)} \mid S)$}
                        \State $\bm{\alpha}^{(i)} \leftarrow \bm{\alpha}^{(i)} - \alpha^{(i)}_m \cdot \bm{e}_m + \alpha^\ast_{m'} \cdot \bm{e}_{m'}$; \Comment{Delete $m$ and insert $m'$}
                        \State \textbf{break};
                    \EndIf
                \EndFor
                \If{$\alpha^{(i)}_m = 0$}
                    \State \textbf{break}; 
                \EndIf
            \EndFor
            \If{$\bm{\alpha}^{(i)} = \bm{\alpha}^{(i-1)}$}
                \State \textbf{break};
            \Else
                \While{not convergence} \Comment{Minimize $G_{\gamma, \lambda}$ with $\gamma = \lambda = 0$}
                    \For{$m \in \supp{\bm{\alpha}^{(i)}}$}
                        \State $\alpha^{(i)}_{m} \leftarrow \argmin_{\alpha_{m} \in \mathbb{R}} L(\bm{\alpha}^{(i)}_{-m} + \alpha_{m} \cdot \bm{e}_{m} \mid S)$; \label{line:finetune}
                    \EndFor
                \EndWhile
            \EndIf
        \EndFor
        % \State \textbf{return} $\bm{\alpha}^{(i)}$;
    \end{algorithmic}
    \label{algo:cdls}
\end{algorithm}

\subsection{Learning Algorithm}
\cref{algo:cdls} presents an algorithm for solving \cref{prob:lire}. 
Our algorithm is based on a coordinate descent algorithm with local search proposed by \citet{Dedieu:JMLR2021} and \citet{Liu:AISTATS2022}. 
Given an initial weight vector $\bm{\alpha}^{(0)}$, which can be efficiently obtained in practice by solving \cref{prob:rulefit}, we iteratively update it until the update converges or the number of iterations reaches a given maximum number $I \in \mathbb{N}$.
% In our experiments, we set $I = 5000$. 
Each iteration $i \in [I]$ consists of the following steps to update the current weight vector $\bm{\alpha}^{(i)}$:
\begin{description}
    \item[Step 1.]
    For each rule $m \in \supp{\bm{\alpha}^{(i)}}$ in the current support, we update its weight $\alpha^{(i)}_m$ so that our learning objective function $G_{\gamma, \lambda}$ is minimized with respect to the coordinate $\alpha_m$ (line~\ref{line:delete}). 
    If the weight $\alpha^{(i)}_m$ is updated to $0$, then we delete $m$ from the support and go to Step 3. 
    \item[Step 2.]
    For each rule $m \in \supp{\bm{\alpha}^{(i)}}$, we attempt to replace $m$ with another rule $m' \in [M] \setminus \supp{\bm{\alpha}^{(i)}}$ outside the support~(line~\ref{line:replace}). 
    For efficiency, when we find a rule $m'$ that improves the objective value, we immediately delete $m$ by setting its weight to $0$ and add $m'$ to the support with a weight $\alpha^\ast_{m'}$. 
    \item[Step 3.]
    If the support of $\bm{\alpha}^{(i)}$ is changed from that of $\bm{\alpha}^{(i-1)}$, we optimize the weight of each rule in $\supp{\bm{\alpha}^{(i)}}$ so that the empirical risk $L$ is minimized (line~\ref{line:finetune}), and go to the next iteration $i+1$. 
\end{description}

\subsection{Analytical Solution to Coordinate Update}
In lines~\ref{line:delete}, \ref{line:replace}, and \ref{line:finetune} of \cref{algo:cdls}, we need to update the weight of each rule $m$ so that our learning objective $G_{\gamma, \lambda}$ is minimized with respect to the coordinate $\alpha_m$. 
We show that we can obtain an analytical solution to this coordinate update problem if we employ the exponential loss $l(y, f_{\bm{\alpha}}(\bm{x})) = e^{-y \cdot f_{\bm{\alpha}}(\bm{x})}$ as the loss function $l$ in $G_{\gamma, \lambda}$. 
As with the previous study on the $\ell_0$-regularized classifier~\citep{Liu:AISTATS2022}, our idea is based on the technique of AdaBoost~\citep{Freund:JCSS1997}, which iteratively updates the weight of each base learner with an analytical solution that minimizes the exponential loss~\citep{Mohri:2012:Foundations}. 
In \cref{theo:analytical}, we extend the technique of AdaBoost to obtain an analytical solution to our coordinate update problem with $G_{\gamma, \lambda}$. 
\begin{theorem}\label{theo:analytical}
For a weight vector $\bm{\alpha} \in \mathbb{R}^M$ and a rule $m \in [M]$ with $\alpha_m = 0$, we consider the coordinate update problem that is formulated as follows:
\begin{align*}
    \alpha^\ast_{m} = \argmin_{\alpha'_{m} \in \mathbb{R}} G_{\gamma, \lambda}(\bm{\alpha} + \alpha'_{m} \cdot \bm{e}_{m} \mid S),
\end{align*}
where $\bm{e}_m = (e_{m, 1}, \dots, e_{m, M}) \in \set{0,1}^M$ is a vector with $e_{m ,m} = 1$ and $e_{m, m'} = 0$ for all $m' \in [M] \setminus \set{m}$. 
If the loss function $l$ in the objective function $G_{\gamma, \lambda}$ is the exponential loss $l(y, f_{\bm{\alpha}}(\bm{x})) = e^{-y \cdot f_{\bm{\alpha}}(\bm{x})}$, then we have
\begin{align*}
    \alpha^\ast_m =
    \begin{cases}
        0 & \text{if } \epsilon_m^- \in [\frac{1}{2}-B_m, \frac{1}{2}+B_m], \\
        \frac{1}{2} \ln\frac{1-\epsilon_m^-}{\epsilon_m^-} & \text{otherwise},
    \end{cases}
\end{align*}
where 
$\epsilon_m^- = \frac{\frac{1}{N} \sum_{n \in [N]: y_n \cdot r_m(\bm{x}_n) = -1} l(y_n, f_{\alpha}(\bm{x}_n))}{\epsilon_m}$, 
$\epsilon_m = \frac{1}{N} \sum_{n \in [N]: r_m(\bm{x}_n) = 1} l(y_n, f_{\bm{\alpha}}(\bm{x}_n))$, 
$B_m = \frac{\sqrt{C_m \cdot (2 \cdot \epsilon_m - C_m)}}{2 \cdot \epsilon_m}$, 
$C_m = \gamma + \lambda \cdot \frac{p_m - \Omega_\mathrm{L}(\bm{\alpha} \mid S)}{1 + |\supp{\bm{\alpha}}|}$, 
and $p_m = \frac{1}{N}\sum_{n=1}^{N} r_m(\bm{x}_n)$. 
\end{theorem}

\cref{theo:analytical} implies that a rule $m$ outside the current support does not improve the objective value if $\epsilon_m^- \in [\frac{1}{2}-B_m, \frac{1}{2}+B_m]$; otherwise, we can update its weight as $\alpha^\ast_m = \frac{1}{2} \ln\frac{1-\epsilon_m^-}{\epsilon_m^-}$. 
By \cref{theo:analytical}, we can solve the coordinate update problem in \cref{algo:cdls} analytically. 
Our proof of \cref{theo:analytical} is shown in Appendix. 

In our experiments, we employed the exponential loss as $l$ for learning LIRE classifiers. 
In addition to the existence of an analytical solution, another advantage of the exponential loss is that we can efficiently compute the objective value of the updated weight vector $\bm{\alpha} + \alpha^\ast_{m} \cdot \bm{e}_{m}$ by simple mathematical operations. 
This is mainly because we can update the empirical risk $L$ in our learning objective $G_{\gamma, \lambda}$ by multiplying each loss term $l(y_n, f_{\bm{\alpha}}(\bm{x}_n))$ by $e^{-y_n \cdot \alpha^\ast_{m} \cdot r_m(\bm{x}_n)}$ if $l$ is the exponential loss~\citep{Mohri:2012:Foundations}. 
Note that the exponential loss is known to perform well similar to the other popular loss functions, such as the logistic loss~\citep{Liu:AISTATS2022}.

\section{Experiments}\label{sec:experiments}
To investigate the performance of our LIRE, we conducted numerical experiments on real datasets. 
All the code was implemented in Python 3.7 with scikit-learn 1.0.2 and is available at \url{https://github.com/kelicht/lire}. 
All the experiments were conducted on Ubuntu 20.04 with Intel Xeon E-2274G 4.0 GHz CPU and 32 GB memory. 

Our experimental evaluation answers the following questions:
(1)~How is the trade-off between the accuracy and interpretability of LIRE compared to RuleFit? 
(2)~How does our local interpretability regularizer affect the accuracy and interpretability of rule ensembles? 
(3)~How is the performance of LIRE compared to the baselines on the benchmark datasets? 
Owing to page limitations, the complete settings and results (e.g., dataset details, hyper-parameter tuning, other accuracy criteria, and statistical tests) are shown in Appendix.

\begin{figure}[t]
    \centering
    \includegraphics[width=\linewidth]{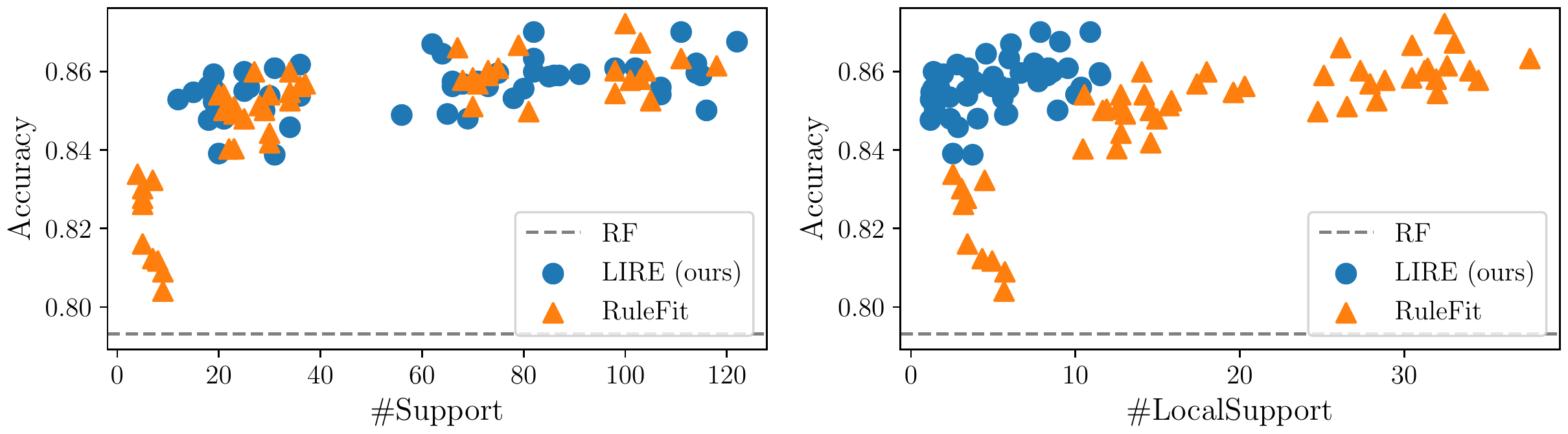}
    \caption{Experimental results on the accuracy-interpretability trade-off analysis. }
    \label{fig:exp:tradeoff}
\end{figure}
\subsection{Accuracy-Interpretability Trade-Off}
First, we examine the trade-off between the accuracy and interpretability of our LIRE compared to RuleFit. 
We used the Adult dataset ($N=32561, D=108$)~\citep{Dua:2019} and conducted $10$-fold cross-validation (CV). 
In each fold, we trained a random forest (RF) with $100$ decision trees and obtained $M=1220.3$ rules on average as $R$. 
For interpretability, each decision tree is trained with a maximum depth of $3$; that is, the length of each rule is less than or equal to $3$. 
Then, we trained rule ensembles by RuleFit and our LIRE and measured their test accuracy, support size (\#Support), and average local support size on the test set (\#LocalSupport). 
To obtain models with different support sizes, we trained multiple models by varying the hyper-parameter $\gamma$. 
For LIRE, we set $\lambda = 1.0$. 

\cref{fig:exp:tradeoff} shows the results, where the left (resp.\ right) figure presents the scatter plot between the test accuracy and support size (resp.\ local support size). 
From \cref{fig:exp:tradeoff}, we can see that LIRE 
(1)~attained a similar trade-off between accuracy and support size to RuleFit, and 
(2)~achieved lower local support size than RuleFit without degrading accuracy. 
These results suggest that our LIRE could obtain more locally interpretable rule ensembles than RuleFit while maintaining similar accuracy and support size. 
Thus, we have confirmed that our method can realize local interpretability without compromising accuracy. 

\begin{figure}[t]
    \centering
    \includegraphics[width=\linewidth]{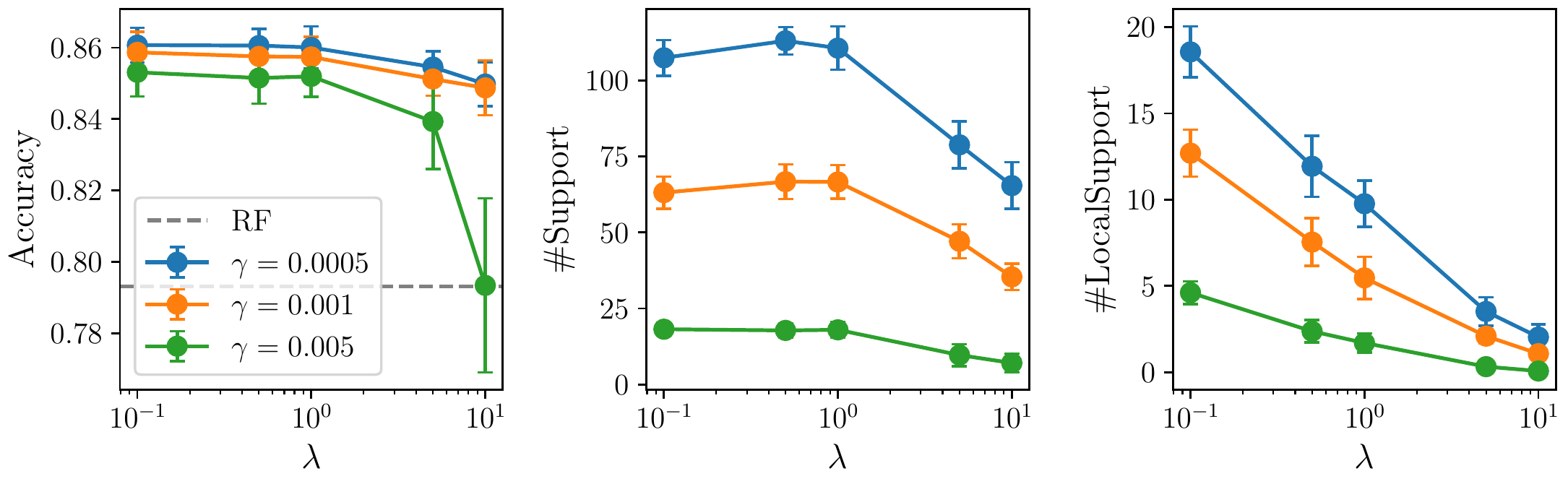}
    \caption{Experimental results on the sensitivity of the trade-off parameter $\lambda$.}
    \label{fig:exp:sensitivity}
\end{figure}
\subsection{Effect of Local Interpretability Regularizer}
Next, we analyze the effect of our local interpretability regularizer $\Omega_\mathrm{L}$ on rule ensembles by varying the hyper-parameter $\lambda$. 
As in the previous experiment, we used the Adult dataset and conducted $10$-fold CV. 
We trained rule ensembles by varying $\lambda$, and measured their average accuracy, support size, and local support size. 
To control the support size, we set $\gamma$ to three different values. 

\cref{fig:exp:sensitivity} shows the average accuracy, support size, and local support size for each $\lambda$. 
For $\gamma \in \set{0.0005, 0.001}$, we could reduce the local support size without significantly degrading accuracy by increasing $\lambda$, which indicates that we could obtain accurate and locally interpretable rule ensembles. 
Furthermore, we can see that the support size also decreased for large $\lambda$ while maintaining accuracy. 
In contrast, for $\gamma = 0.005$, the average accuracy decreased when $\lambda > 1.0$. 
This result suggests that our local interpretability regularizer $\Omega_\mathrm{L}$ may harm accuracy when $\gamma$ is large, i.e., the support size is small. 
Thus, to maintain accuracy with $\Omega_\mathrm{L}$, we need to keep the support size larger to some extent by setting $\gamma$ to be smaller. 
These observations give us insight into the choice of $\lambda$ and $\gamma$ in practice. 
% For $\gamma \in \set{0.001, 0.005}$, we could maintain similar accuracy regardless of $\lambda$, while reducing the average support size and local support size. 
% For example, the average accuracy, support size, and local support size for $\gamma = 0.001$ and $\lambda = 1.0$ were respectively 85.7\%, 66.6, and 5.4, which indicates that we could obtain an accurate and locally interpretable rule ensemble. 

\begin{figure}[t]
    \centering
    \includegraphics[width=0.9\linewidth]{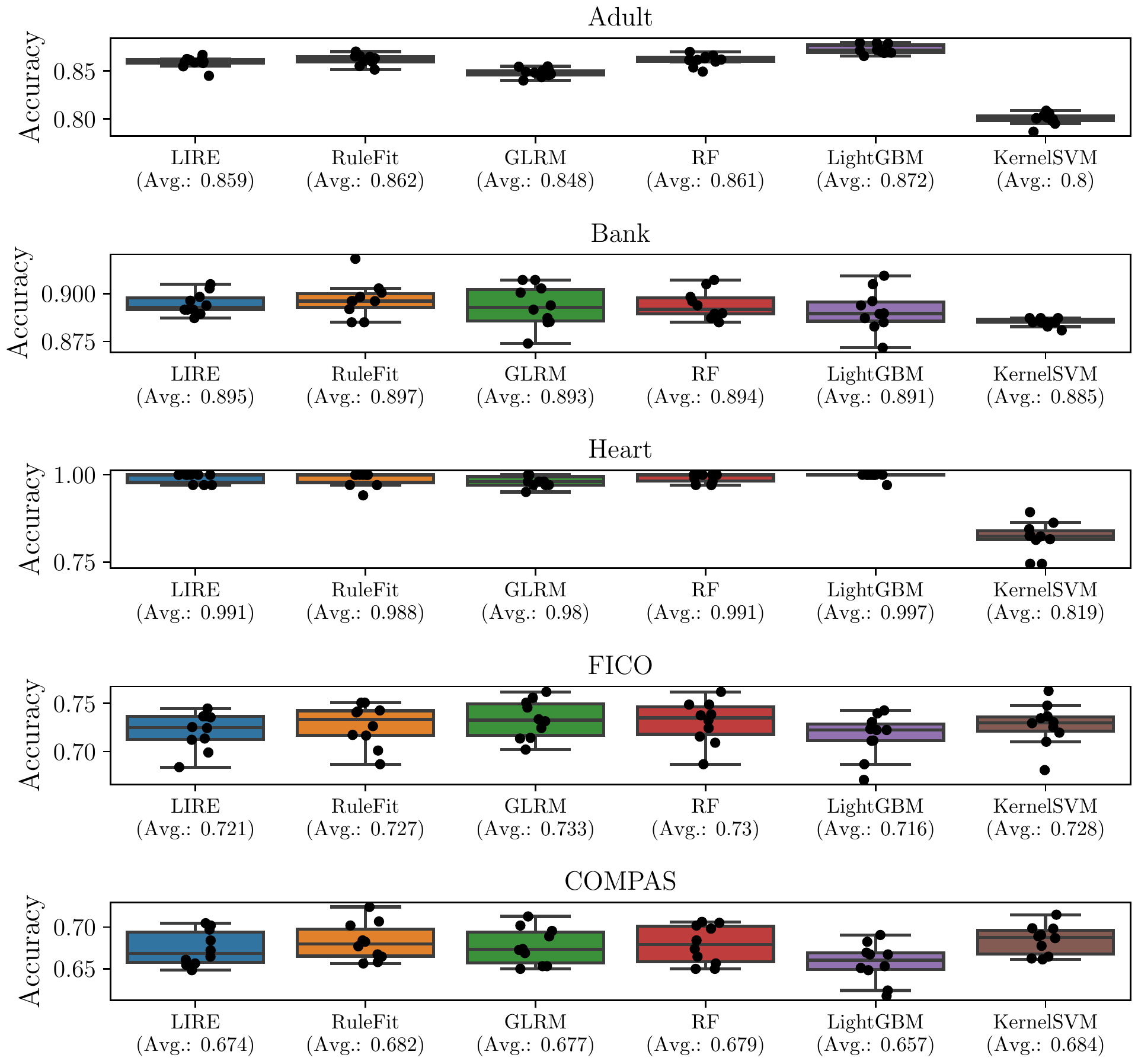}
    \caption{Experimental results on the test accuracy in 10-fold cross-validation. }
    \label{fig:exp:comparison}
\end{figure}
\begin{figure}[t]
    \centering
    \includegraphics[width=0.9\linewidth]{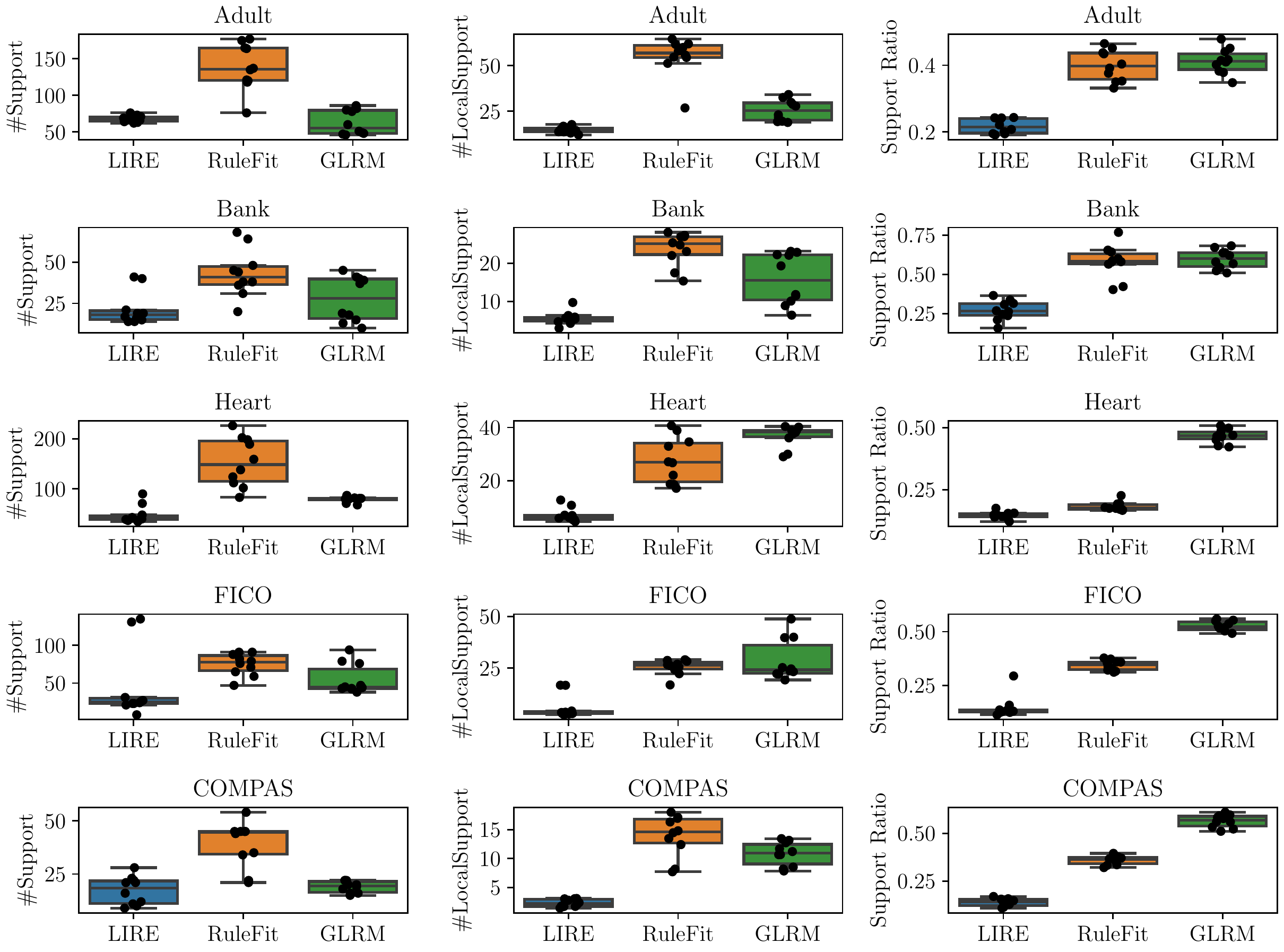}
    \caption{Experimental results on the support size in 10-fold cross-validation. }
    \label{fig:exp:support}
\end{figure}

\subsection{Performance Comparison}
Finally, we evaluate the performance of our LIRE on benchmark datasets in comparison with the existing methods. 
We used five datasets: three datasets are Adult, Bank, and Heart from the UCI repository~\citep{Dua:2019}, and two datasets are FICO~\citep{fico:2018} and COMPAS~\citep{compas:2016}. 
In addition to RuleFit, we compared LIRE with the generalized linear rule models~(GLRM)~\citep{Wei:ICML2019}, another existing method for learning rule ensembles by column generation. 
We also employed three complex models as baselines: RF~\citep{Breiman:ML2001}, LightGBM~\citep{Ke:NIPS2017}, and KernelSVM. 
In each fold, we tuned the hyper-parameters of each method through hold-out validation. 

\cref{fig:exp:comparison} shows the test accuracy of each method in $10$-fold CV. 
From \cref{fig:exp:comparison}, we can see that LIRE achieved comparable accuracy to the other rule ensembles, as well as complex models, regardless of the datasets. 
\cref{fig:exp:support} shows the results on the support size, local support size, and support ratio, which is defined as the ratio of \#LocalSupport to \#Support for each rule ensemble. 
We can see that LIRE stably achieved lower local support sizes and support ratios than RuleFit and GLRM. 
Furthermore, LIRE also achieved lower support sizes than RuleFit while maintaining similar accuracy, which may be caused by the effect of our global interpretability regularizer $\Omega_\mathrm{G}$, i.e., $\ell_0$-regularization. 
These results indicate that our LIRE achieved local interpretability while maintaining not only accuracy but also support size comparable to the baselines. 
Therefore, we have confirmed that we can learn more locally interpretable rule ensembles than the baselines without degrading accuracy in the benchmark datasets. 
% While there is no significant difference in the support size between the methods, the local support size and support ratio of LIRE were stably lower than RuleFit and GLRM. 

Regarding the computational time shown in Appendix, LIRE was slower than the baselines because its objective $G_{\gamma, \lambda}$ includes regularizers that have combinatorial nature. 
For example, the average computation time of LIRE, RuleFit, and GLRM on the Adult was $\bm{148.5}$, $\bm{9.658}$, and $\bm{75.93}$ seconds, respectively. 
However, \cref{fig:exp:comparison,fig:exp:support} indicate that LIRE achieved higher local interpretability than the baselines without degrading accuracy within a few minutes, even for the dataset with $N > 30000$ and the size of candidate rules with $M > 1000$.

\section{Related Work}
% In this section, we discuss our contributions to the communities in comparison with the existing studies. 

\subsection{Globally Interpretable Models}
This paper mainly relates to the communities of interpretable machine learning~\citep{Rudin:SS2022}. 
With the emerging trend of leveraging machine learning models in various high-stakes decision-making tasks, interpretable models, such as sparse linear models~\citep{Dedieu:JMLR2021,Liu:AISTATS2022,Ustun:JMLR2019} and rule models~\citep{Angelino:KDD2017,Lakkaraju:KDD2016,Hu:NIPS2019,Dash:NIPS2018,Yang:ICML2017}, have attracted increasing attention in recent years. 
Rule ensembles, also known as generalized linear rule models, are one of the popular rule models based on the linear combination of weighted rules~\citep{Friedman:AAS2008,Wei:ICML2019,Nalenz:AISTATS2022,Benard:AISTATS2021,Eckstein:ICML2017,Nakagawa:KDD2016,Kato:PAMI2023}. 

In general, rule ensembles have a trade-off between their accuracy and interpretability. 
To achieve good generalization, rule ensembles often need to include a sufficiently large number of weighted rules~\citep{Nalenz:AISTATS2022}. 
However, increasing the total number of weighted rules degrades the interpretability of a model because it makes the entire model hard for human users to understand~\citep{Rudin:NMI2019,Freitas:EN2014,Lipton:Queue2018,Doshi-Velez:arxiv2017}. 
To address this trade-off, most of the existing methods focus on achieving accuracy with as few weighted rules as possible through $\ell_1$-regularization~\citep{Friedman:AAS2008,Wei:ICML2019,Eckstein:ICML2017,Nakagawa:KDD2016,Kato:PAMI2023}. 

Our contribution is to propose another approach for addressing the accuracy-interpretability trade-off of rule ensembles.
We introduced a new concept of interpretability for a rule ensemble model, named local interpretability, focusing on its individual predictions rather than the model itself. 
Our concept has a similar spirit to the falling rule lists~\citep{Wang:AISTATS2015} and locally sparse neural networks~\citep{Yang:ICML2022} that can explain individual predictions in an interpretable manner, and is helpful in some practical situations where we need to explain undesired predictions for individual users, such as loan approvals and medical diagnoses~\citep{Rudin:NMI2019,Caruana:KDD2015,Zhang:DMKD2023}. 
We also empirically confirmed that we could learn more locally interpretable models than the existing methods while achieving comparable accuracy. 
Note that our framework can be combined with the existing practical techniques of rule ensembles, such as stabilization~\citep{Benard:AISTATS2021} and compression~\citep{Nalenz:AISTATS2022} of weighted rules.

\subsection{Local Explanation Methods}
% A recent major approach to realizing interpretability is post-hoc methods that extract local explanations of the individual predictions made by a learned model. 
Our approach is inspired by the recent methods that extract local explanations of the individual predictions made by a learned model. 
These methods provide local explanations in a post-hoc manner by locally approximating the decision boundary of a model by the linear models~\citep{Ribeiro:KDD2016,Lundberg:NIPS2017} or rule sets~\citep{Ribeiro:AAAI2018,Rudin:JMLR2023}. 
To improve the quality of local explanations, some papers have proposed to regularize a model during its training so that we can obtain better local explanations in terms of their local approximation fidelity~\citep{Plumb:NIPS2020} or consistency with domain knowledge~\citep{Ross:IJCAI2017,Rieger:ICML2020}. 
Our proposed method also regularizes a rule ensemble for the quality of local explanations, i.e., the total number of weighted rules used for making each individual prediction, during its training. 

While several post-hoc local explanation methods have been proposed, recent studies pointed out the issue of their faithfulness to an underlying model~\citep{Jacovi:ACL2020,Yoon:TMLR2022}.
Most of the existing methods have a risk that their explanations are inconsistent with the actual behavior of the model because they generate explanations by local approximation~\citep{Rudin:NMI2019,Alvarez-Melis:WHI2018}. 
In contrast to them, the local explanations provided by our method are faithful to the model since they consist of the weighted rules included in the model actually and they are provided without approximation.

\section{Conclusion}
In this paper, we proposed a new framework for learning rule ensembles, named locally interpretable rule ensemble (LIRE), that simultaneously achieves accuracy and interpretability. 
We introduced a new criterion of interpretability, named local interpretability, as the total number of rules that are necessary to explain individual predictions made by the model rather than to explain the model itself. 
Then, we proposed a regularizer that promotes the local interpretability of a rule ensemble, and developed an efficient learning algorithm with the regularizer by coordinate descent with local search. 
By experiments, we confirmed that our method learns more locally interpretable rule ensembles than the existing methods, such as RuleFit, while attaining comparable accuracy. 

\paragraph{Limitations and Future Work.}
There are several future directions to improve our LIRE. 
First, a theoretical analysis of the convergence property of \cref{algo:cdls} is essential to developing a more efficient one~\citep{Dedieu:JMLR2021}. 
We also need to analyze the sensitivity of the heyper-parameters $\lambda$ and $\gamma$ in more detail to decide their default values. 
Second, it is important to conduct user studies to evaluate our local interpretability in real applications~\citep{Doshi-Velez:arxiv2017,Lage:HCOMP2019}. 
Finally, extending our local interpretability to other rule models, such as decision trees, is interesting for future work~\citep{Zhang:DMKD2023}.

% \newpage
\section*{Acknowledgement}
We wish to thank Koji Maruhashi, Takuya Takagi, Ken Kobayashi, and Yuichi Ike for making a number of valuable suggestions.
We also thank the anonymous reviewers for their insightful comments.

\bibliographystyle{abbrvnat}
\begin{small}
    \bibliography{ref}
\end{small}

\newpage
\appendix
\section{Proof of Theorem~1}
% \begin{theorem}\label{theo:analytical}
% For a weight vector $\bm{\alpha} \in \mathbb{R}^M$ and a rule $m \in [M]$ with $\alpha_m = 0$, we consider the coordinate update problem that is formulated as follows:
% \begin{align*}
%     \alpha^\ast_{m} = \argmin_{\alpha'_{m} \in \mathbb{R}} G_{\gamma, \lambda}(\bm{\alpha} + \alpha'_{m} \cdot \bm{e}_{m} \mid S),
% \end{align*}
% where $\bm{e}_m = (e_{m, 1}, \dots, e_{m, M}) \in \set{0,1}^M$ is a vector with $e_{m ,m} = 1$ and $e_{m, m'} = 0$ for all $m' \in [M] \setminus \set{m}$. 
% If the loss function $l$ in the objective function $G_{\gamma, \lambda}$ is the exponential loss $l(y, f_{\bm{\alpha}}(\bm{x})) = e^{-y \cdot f_{\bm{\alpha}}(\bm{x})}$, then we have
% \begin{align*}
%     \alpha^\ast_m =
%     \begin{cases}
%         0 & \text{if } \epsilon_m^- \in [\frac{1}{2}-B_m, \frac{1}{2}+B_m], \\
%         \frac{1}{2} \ln\frac{1-\epsilon_m^-}{\epsilon_m^-} & \text{otherwise},
%     \end{cases}
% \end{align*}
% where 
% $\epsilon_m^- = \frac{\frac{1}{N} \sum_{n \in [N]: y_n \cdot r_m(\bm{x}_n) = -1} l(y_n, f_{\bm{\alpha}}(\bm{x}_n))}{\epsilon_m}$, 
% $\epsilon_m = \frac{1}{N} \sum_{n \in [N]: r_m(\bm{x}_n) = 1} l(y_n, f_{\bm{\alpha}}(\bm{x}_n))$, 
% $B_m = \frac{\sqrt{C_m \cdot (2 \cdot \epsilon_m - C_m)}}{2 \cdot \epsilon_m}$, 
% $C_m = \gamma + \lambda \cdot \frac{p_m - \Omega_\mathrm{L}(\bm{\alpha} \mid S)}{1 + |\supp{\bm{\alpha}}|}$, 
% and $p_m = \frac{1}{N}\sum_{n=1}^{N} r_m(\bm{x}_n)$. 
% \end{theorem}
% 
To prove \cref{theo:analytical}, we show the following lemma. 
\begin{lemma}\label{lemm:increase}
Let $\bm{\alpha} \in \mathbb{R}^M$ be a weight vector with $\alpha_m = 0$ for a rule $m \in [M]$. 
We write $\Omega(\bm{\alpha}) \coloneqq \gamma \cdot \Omega_\mathrm{G}(\bm{\alpha}) + \lambda \cdot \Omega_\mathrm{L}(\bm{\alpha} \mid S)$. 
Then, for any $\alpha'_m \not= 0$, we have
\begin{align*}
    \Omega(\bm{\alpha} + \alpha'_{m} \cdot \bm{e}_{m}) - \Omega(\bm{\alpha})
    = \gamma + \lambda \cdot \frac{p_m - \Omega_\mathrm{L}(\bm{\alpha} \mid S)}{1 + |\supp{\bm{\alpha}}|},
\end{align*}
where $\bm{e}_m = (e_{m, 1}, \dots, e_{m, M}) \in \set{0,1}^M$ is a vector with $e_{m ,m} = 1$ and $e_{m, m'} = 0$ for all $m' \in [M] \setminus \set{m}$, and $p_m = \frac{1}{N}\sum_{n=1}^{N} r_m(\bm{x}_n)$. 
\end{lemma}
\begin{proof}
For the global interpretability regularizer $\Omega_\mathrm{G}$, we have $\Omega_\mathrm{G}(\bm{\alpha} + \alpha'_{m} \cdot \bm{e}_{m}) - \Omega_\mathrm{G}(\bm{\alpha}) = |\supp{\bm{\alpha} + \alpha'_{m} \cdot \bm{e}_{m}}|- |\supp{\bm{\alpha}}| = 1$. 
Let $K \coloneqq |\supp{\bm{\alpha}}|$ be the support size of $\alpha$. 
For the local interpretability regularizer $\Omega_\mathrm{L}$, we have
\begin{align*}
    \Omega_\mathrm{L}(\bm{\alpha} + \alpha'_{m} \cdot \bm{e}_{m} \mid S) - \Omega_\mathrm{L}(\bm{\alpha} \mid S) &= \frac{1}{N} \sum_{n=1}^{N} \frac{|\lsupp{\bm{\alpha} + \alpha'_{m} \cdot \bm{e}_{m} \mid \bm{x}_n}|}{|\supp{\bm{\alpha} + \alpha'_{m} \cdot \bm{e}_{m}}|} - \frac{1}{N} \sum_{n=1}^{N} \frac{|\lsupp{\bm{\alpha} \mid \bm{x}_n}|}{|\supp{\bm{\alpha}}|} \\
    &= \frac{1}{N} \sum_{n=1}^{N} \left( \frac{|\lsupp{\bm{\alpha} + \alpha'_{m} \cdot \bm{e}_{m} \mid \bm{x}_n}|}{1 + K} - \frac{|\lsupp{\bm{\alpha} \mid \bm{x}_n}|}{K} \right) \\
    &= \frac{1}{N} \sum_{n=1}^{N} \frac{(|\lsupp{\bm{\alpha} \mid \bm{x}_n}| + r_m(\bm{x}_n)) \cdot K - |\lsupp{\bm{\alpha} \mid \bm{x}_n}| \cdot (1+K)}{(1 + K) \cdot K} \\
    &= \frac{1}{N} \sum_{n=1}^{N} \frac{r_m(\bm{x}_n) \cdot K - |\lsupp{\bm{\alpha} \mid \bm{x}_n}|}{(1 + K) \cdot K} \\
    &= \frac{\frac{1}{N} \sum_{n=1}^{N} r_m(\bm{x}_n) - \frac{1}{N} \sum_{n=1}^{N} \frac{|\lsupp{\bm{\alpha} \mid \bm{x}_n}|}{K}}{1 + K} = \frac{p_m - \Omega_\mathrm{L}(\bm{\alpha} \mid S)}{1 + |\supp{\bm{\alpha}}|}. 
\end{align*}
By combining these results, we obtain $\Omega(\bm{\alpha} + \alpha'_{m} \cdot \bm{e}_{m}) - \Omega(\bm{\alpha}) = \gamma + \lambda \cdot \frac{p_m - \Omega_\mathrm{L}(\bm{\alpha} \mid S)}{1 + |\supp{\bm{\alpha}}|}$. 
% \qed
\end{proof}

Using \cref{lemm:increase}, we give a proof of \cref{theo:analytical} as follows. 
\begin{proof}[\cref{theo:analytical}]
Recall that our learning objective function $G_{\gamma, \lambda}$ is defined as $G_{\gamma, \lambda}(\bm{\alpha} \mid S) = L(\bm{\alpha} \mid S) + \gamma \cdot \Omega_\mathrm{G}(\bm{\alpha}) + \lambda \cdot \Omega_\mathrm{L}(\bm{\alpha} \mid S)$. 
From \cref{lemm:increase}, when we set $\alpha_m$ to a non-zero value, the value of the regularizers in $G_{\gamma, \lambda}$ increases by $C_m \coloneqq \gamma + \lambda \cdot \frac{p_m - \Omega_\mathrm{L}(\bm{\alpha} \mid S)}{1 + |\supp{\bm{\alpha}}|}$. 
Hence, if the maximum possible decrease in the empirical risk $L(\bm{\alpha} + \alpha'_{m} \cdot \bm{e}_{m} \mid S)$ by setting $\alpha'_m \not= 0$ is less than $C_m$, then we have $\alpha^\ast_m = 0$. 
In the following, we show the condition under which this occurs. 
% In the following, we show that we can obtain (1) an analytical solution $\alpha'_{m}$ that minimizes the empirical risk and (2) the condition where we can set $\alpha^\ast_m = 0$ by comparing the maximum possible decrease in the empirical risk by setting $\alpha^\ast_m = \alpha'_{m}$ and $C_m$. 

First, we consider the optimal value $\alpha'_m$ that minimizes the empirical risk $L(\bm{\alpha} + \alpha'_{m} \cdot \bm{e}_{m} \mid S)$. 
As with AdaBoost, we can obtain the analytical solution for $\alpha'_{m}$ by solving $\frac{\partial}{\partial \alpha_m} L(\bm{\alpha} \mid S) = 0$. 
For $m$ and $n \in [N]$, we write $z_{nm} = y_n \cdot r_m(\bm{x}_n)$. 
Since the loss function is the exponential loss $l(y, f_{\bm{\alpha}}(\bm{x})) = e^{-y \cdot f_{\bm{\alpha}}(\bm{x})}$, we have
\begin{align*}
    \frac{\partial}{\partial \alpha_m} L(\bm{\alpha} \mid S) 
    &= \frac{1}{N} \sum_{n=1}^{N} - z_{nm} \cdot e^{-y_n \cdot f_{\bm{\alpha}}(\bm{x})} \cdot e^{- z_{nm} \cdot \alpha_m} \\
    &= \frac{1}{N} \sum_{n \in [N]: z_{nm} \not= 0} - z_{nm} \cdot e^{-y_n \cdot f_{\bm{\alpha}}(\bm{x})} \cdot e^{- z_{nm} \cdot \alpha_m} \\
    &= \frac{1}{N} \left( \sum_{n \in [N]: z_{nm} = +1} - e^{-y_n \cdot f_{\bm{\alpha}}(\bm{x})} \cdot e^{- \alpha_m} + \sum_{n \in [N]: z_{nm} = -1} e^{-y_n \cdot f_{\bm{\alpha}}(\bm{x})} \cdot e^{\alpha_m} \right). \\
\end{align*}
Let $\epsilon_m \coloneqq \frac{1}{N} \sum_{n \in [N]: z_{nm} \not= 0} e^{-y_n \cdot f_{\bm{\alpha}}(\bm{x})}$ be the normalizing constant, and we define
\begin{align*}
    \epsilon_m^+ \coloneqq \frac{\frac{1}{N} \sum_{n \in [N]: z_{nm} = +1} e^{-y_n \cdot f_{\bm{\alpha}}(\bm{x})}}{\epsilon_m} 
    \;\;\text{and}\;\;
    \epsilon_m^- \coloneqq \frac{\frac{1}{N} \sum_{n \in [N]: z_{nm} = -1} e^{-y_n \cdot f_{\bm{\alpha}}(\bm{x})}}{\epsilon_m}. 
\end{align*}
By definition, $\epsilon_m^+ + \epsilon_m^- = 1$ holds. 
Then, we obtain the analytical solution $\alpha'_m$ to $\frac{\partial}{\partial \alpha_m} L(\bm{\alpha} \mid S) = 0$ by
\begin{align*}
    \frac{\partial}{\partial \alpha_m} L(\bm{\alpha} \mid S) = 0
    &\iff - \epsilon_m^+ \cdot e^{- \alpha'_m} + \epsilon_m^- \cdot e^{\alpha'_m} = 0 \\
    &\iff \alpha'_m = \frac{1}{2} \ln \frac{\epsilon_m^+}{\epsilon_m^-} = \frac{1}{2} \ln \frac{1 - \epsilon_m^-}{\epsilon_m^-}. 
\end{align*}

Using $\alpha'_m$, we can compute the maximum possible decrease in the empirical risk as follows:
\begin{align*}
    &L(\bm{\alpha} \mid S) - L(\bm{\alpha} + \alpha'_{m} \cdot \bm{e}_{m} \mid S) \\
    &= \frac{1}{N} \sum_{n \in [N]: z_{nm} \not= 0} e^{-y_n \cdot f_{\bm{\alpha}}(\bm{x})} \cdot (1 - e^{- z_{nm} \cdot \alpha'_m}) \\
    &= \frac{1}{N} \left( \sum_{n \in [N]: z_{nm} = +1} e^{-y_n \cdot f_{\bm{\alpha}}(\bm{x})} \cdot \left(1 - \sqrt{\frac{\epsilon_m^-}{\epsilon_m^+}}\right) + \sum_{n \in [N]: z_{nm} = -1}e^{-y_n \cdot f_{\bm{\alpha}}(\bm{x})} \cdot \left(1 - \sqrt{\frac{\epsilon_m^+}{\epsilon_m^-}}\right) \right) \\
    &= \epsilon_m \cdot \left( \epsilon_m^+ \cdot \left(1 - \sqrt{\frac{\epsilon_m^-}{\epsilon_m^+}}\right) + \epsilon_m^- \cdot \left(1 - \sqrt{\frac{\epsilon_m^+}{\epsilon_m^-}}\right) \right) \\
    &= \epsilon_m \cdot \left( 1 - 2 \cdot \sqrt{\epsilon_m^- \cdot (1 - \epsilon_m^-)} \right).
\end{align*}
Now, we derive the condition of $\alpha'_m = 0$ by comparing $C_m$ and $L(\bm{\alpha} \mid S) - L(\bm{\alpha} + \alpha'_{m} \cdot \bm{e}_{m} \mid S)$ as follows: 
\begin{align*}
    L(\bm{\alpha} \mid S) - L(\bm{\alpha} + \alpha'_{m} \cdot \bm{e}_{m} \mid S) \leq C_m
    &\iff \epsilon_m \cdot \left( 1 - 2 \cdot \sqrt{\epsilon_m^- \cdot (1 - \epsilon_m^-)} \right) \leq C_m \\
    &\iff (\epsilon_m^-)^2 - \epsilon_m^- + \left( \frac{\epsilon_m- C_m}{2 \cdot \epsilon_m} \right) \leq 0 \\
    &\iff \epsilon_m^- \in \left[ \frac{1}{2} - B_m, \frac{1}{2} + B_m \right],
\end{align*}
where $B_m \coloneqq \frac{\sqrt{C_m \cdot (2 \cdot \epsilon_m - C_m)}}{2 \cdot \epsilon_m}$. 
Therefore, if $\epsilon_m^- \in [ \frac{1}{2} - B_m, \frac{1}{2} + B_m ]$, we have $\alpha^\ast_m = 0$; otherwise, we have $\alpha^\ast_m = \frac{1}{2} \ln \frac{1 - \epsilon_m^-}{\epsilon_m^-}$, which concludes the proof.  
% \qed
\end{proof}

\section{Complete Experimental Settings and Results}
In this section, we present the complete settings and results of our experiments. 
\cref{tab:appendix:exp:datasets} shows the sample size $N$, the total number of features $D$, and the average number of rules $M$ used for learning RuleFit and LIRE. 
In each dataset, all the categorical features were transformed into binary vectors through one-hot encoding in advance. 
Note that the settings of the hyper-parameters used to learn the baseline methods and our method are summarized in \cref{tab:appendix:exp:parameters}. 
For our method, we set the maximum number of iterations in Algorithm~1 by $I = 5000$. 

\subsection{Accuracy-Interpretability Trade-Off}
We present the complete results of the trade-off analyses between the accuracy and interpretability of our LIRE compared to RuleFit. 
In addition to accuracy, we measured the F1 score and area under the ROC curve (AUC) on each test set of $10$-fold cross-validation. 
\cref{fig:appendx:exp:tradeoff1,fig:appendx:exp:tradeoff2,fig:appendx:exp:tradeoff3} show the results with $\lambda = 0.1$, $0.5$, and $1.0$, respectively. 

\subsection{Effect of Local Interpretability Regularizer}
We present the complete results on the effects of our local interpretability regularizer on rule ensembles by varying the trade-off parameter $\lambda$. 
\cref{fig:appendix:exp:sensitivity} shows the results on the average test accuracy, average F1 score, AUC, support size, local support size, and support ratio on each test set. 

\subsection{Performance Comparison}
We present the complete results of the performance comparison with the existing methods. 
In each fold of $10$-fold cross-validation, we tuned the hyper-parameters of each method by hold-out validation.
\cref{tab:appendix:exp:parameters} shows the value ranges of the hyper-parameters of each method. 

\cref{tab:appendix:exp:time} presents the average running time of each method. 
\cref{fig:appendix:exp:accuracy,fig:appendix:exp:f1,fig:appendix:exp:auc} show the results on the accuracy, F1 score, and AUC of each method in $10$-fold cross-validation, respectively. 
\cref{tab:appendix:exp:test:acc,tab:appendix:exp:test:sup} present the p-values of the Wilcoxon signed-rank test between LIRE and the baselines. 

% \newpage
\section*{Ethical Statement}
% Ethics is one of the most important topics to emerge in machine learning and data mining. We ask you to think about the ethical implications of your submission such as, e.g., related to the collection and processing of personal data, the inference of personal information, or the potential use of your work for policing or the military. which will be taken into consideration by the reviewers. As part of your submission, you are asked to include an ethical statement up to one page in length that discusses any ethical implications of your work.

\subsection*{Existing Assets}
All datasets used in \cref{sec:experiments} are publicly available and do not contain any identifiable information or offensive content. 
As they are accompanied by appropriate citations in the main body, see the corresponding references for more details.
Scikit-learn 1.0.2 \footnote{\url{https://scikit-learn.org/stable/}} is publicly available under the BSD-3-Clause license. 
All the scripts and datasets used in our experiments are available in our GitHub repository at \url{https://github.com/kelicht/lire}. 

\subsection*{Potential Impacts}
\paragraph{Positive Impacts.}
Our proposed method, named locally interpretable rule ensemble (LIRE), is a new framework for learning rule ensemble models. 
Our LIRE can explain its individual predictions with a few weighted logical rules, i.e., in a transparent manner for human users. 
Therefore, our LIRE helps decision-makers to validate the prediction results made by machine learning models and ensure the transparency of their decision results in critical tasks, such as loan approvals, medical diagnoses, and judicial decisions~\citep{Rudin:NMI2019,Rudin:SS2022}. 

\paragraph{Negative Impacts.}
Because our method provides explanations of its individual predictions in a transparent manner, one might use the output to extract sensitive information from the training dataset. 
Note, however, that such unintended use can occur not only with our method but also with other interpretable models. 
One possible way to mitigate this risk is to check whether the features in a dataset used to construct the weighted rules might reveal sensitive information before training the model on the dataset and deploying it publicly. 

\subsection*{Limitations}
In the real applications of our LIRE, there might exist three limitations. 
First, as mentioned in our experiments, the computational time of LIRE was certainly longer than those of the baselines. 
To overcome this limitation, we need to analyze the convergence property of our learning algorithm and develop a more efficient one. 
Second, since LIRE has two hyper-parameters, $\lambda$ and $\gamma$, users need to determine these values depending on the dataset by themselves, which may incur additional computational costs. 
Finally, the effectiveness of our local interpretability in real situations has not yet been verified. 
We plan to conduct user studies to evaluate the usefulness of our concept for human users and analyze how much local support size is acceptable for humans~\citep{Lage:HCOMP2019,Doshi-Velez:arxiv2017}.

\begin{table}[p]
    \centering
    \caption{Details of the datasets used in the experiments. }
    \begin{tabular}{
        lccc
        % p{0.15\textwidth}>{\centering}        
        % p{0.2\textwidth}>{\centering}
        % p{0.2\textwidth}>{\centering\arraybackslash}
        % p{0.275\textwidth}
    }
    \toprule
        \textbf{Dataset} & \textbf{\#Examples} $N$ & \textbf{\#Features} $D$ & \textbf{Average \#Rules} $M$ \\
    \midrule
        Adult & $32561$ & $108$ & $1220.3$ \\
        Bank & $4521$ & $35$ & $1225.8$ \\
        Heart & $1025$ & $15$ & $1053.4$ \\
        FICO & $9871$ & $23$ & $1249.8$ \\
        COMPAS & $6167$ & $14$ & $1092.9$ \\
    \bottomrule
    \end{tabular}
    \label{tab:appendix:exp:datasets}
\end{table}

\begin{table}[p]
    \centering
    \caption{
        Details of hyper-parameter tuning for each method. 
        For GLRM, following its original paper, we set $\lambda_0$ and $\lambda_1$ so that $\lambda_1 / \lambda_0 = 0.2$. 
        For RF and LightGBM, $T$ and $L$ denote the maximum numbers of decision trees and leaves in each tree, respectively. 
        For KernelSVM, we used the Gaussian kernel with a regularization parameter $C$. 
        In each hold of CV, we selected the hyper-parameters of each method by hold-out validation with AUC score. 
    }
    \begin{tabular}{lc}
    \toprule
        \textbf{Method} & \textbf{Range of hyper-parameters}  \\
    \midrule
        LIRE & $\gamma \in \set{0.001, 0.002, 0.003, 0.004, 0.005}$, $\lambda = 2 \cdot \gamma$ \\
        RuleFit & $\gamma \in \set{0.0005, 0.00075, 0.001, 0.0025, 0.005, 0.0075, 0.01, 0.025, 0.05}$. \\
        GLRM & $\lambda_0 \in \set{0.00005, 0.0001, 0.0005, 0.001, 0.005, 0.01, 0.05, 0.1}$, $\lambda_1 = 0.2 \cdot \lambda_0$ \\
        RF & $T \in \set{100, 300, 500}$, $L \in \set{64, 128, 256}$ \\
        LightGBM & $T \in \set{100, 300, 500}$, $L \in \set{64, 128, 256}$ \\
        KernelSVM & $C \in \set{ 0.01, 0.05, 0.1, 0.5, 1.0, 5.0, 10.0, 50.0, 100.0}$ \\
    \bottomrule
    \end{tabular}
    \label{tab:appendix:exp:parameters}
\end{table}
\begin{table}[p]
    \centering
    % \scriptsize
    \caption{Experimental results on the average computational time of each method in $10$-fold cross-validation [s]. }
    \begin{tabular}{
        lcccccc
        % p{0.125\textwidth}>{\centering}        
        % p{0.13\textwidth}>{\centering}
        % p{0.13\textwidth}>{\centering}
        % p{0.13\textwidth}>{\centering}
        % p{0.13\textwidth}>{\centering}
        % p{0.13\textwidth}>{\centering\arraybackslash}
        % p{0.13\textwidth}
    }
        \toprule
        \multirow{2}{*}{\textbf{Dataset}} & \multicolumn{3}{c}{\textbf{Rule Ensemble}} & \multicolumn{3}{c}{\textbf{Complex Model}} \\
        \cmidrule(lr){2-4} \cmidrule(lr){5-7}
        & {LIRE (ours)} & {RuleFit} & {GLRM} & {RF} & {LightGBM} & {KernelSVM} \\
        \midrule
        Adult & $148.5 \pm 0.01$ & $9.658 \pm 0.01$ & $75.93 \pm 0.0$ & $5.137 \pm 0.01$ & $0.192 \pm 0.0$ & $401.7 \pm 0.01$ \\
        Bank & $8.446 \pm 0.01$ & $1.308 \pm 0.01$ & $11.78 \pm 0.01$ & $0.825 \pm 0.01$ & $0.252 \pm 0.01$ & $2.017 \pm 0.0$ \\
        Heart & $4.859 \pm 0.01$ & $0.782 \pm 0.02$ & $32.37 \pm 0.02$ & $0.095 \pm 0.01$ & $0.067 \pm 0.01$ & $0.106 \pm 0.05$ \\
        FICO & $74.59 \pm 0.02$ & $1.93 \pm 0.02$ & $14.17 \pm 0.02$ & $2.142 \pm 0.02$ & $0.284 \pm 0.02$ & $10.99 \pm 0.02$ \\
        COMPAS & $9.875 \pm 0.02$ & $1.191 \pm 0.02$ & $2.528 \pm 0.02$ & $0.492 \pm 0.02$ & $0.084 \pm 0.02$ & $4.47 \pm 0.02$ \\
        \midrule
        \textbf{Average} & $49.36$ & $2.974$ & $27.36$ & $1.744$ & $0.176$ & $83.86$ \\
        \bottomrule
    \end{tabular}
    \label{tab:appendix:exp:time}
\end{table}

\newpage

\begin{figure}[p]
    \centering
    \includegraphics[width=0.95\linewidth]{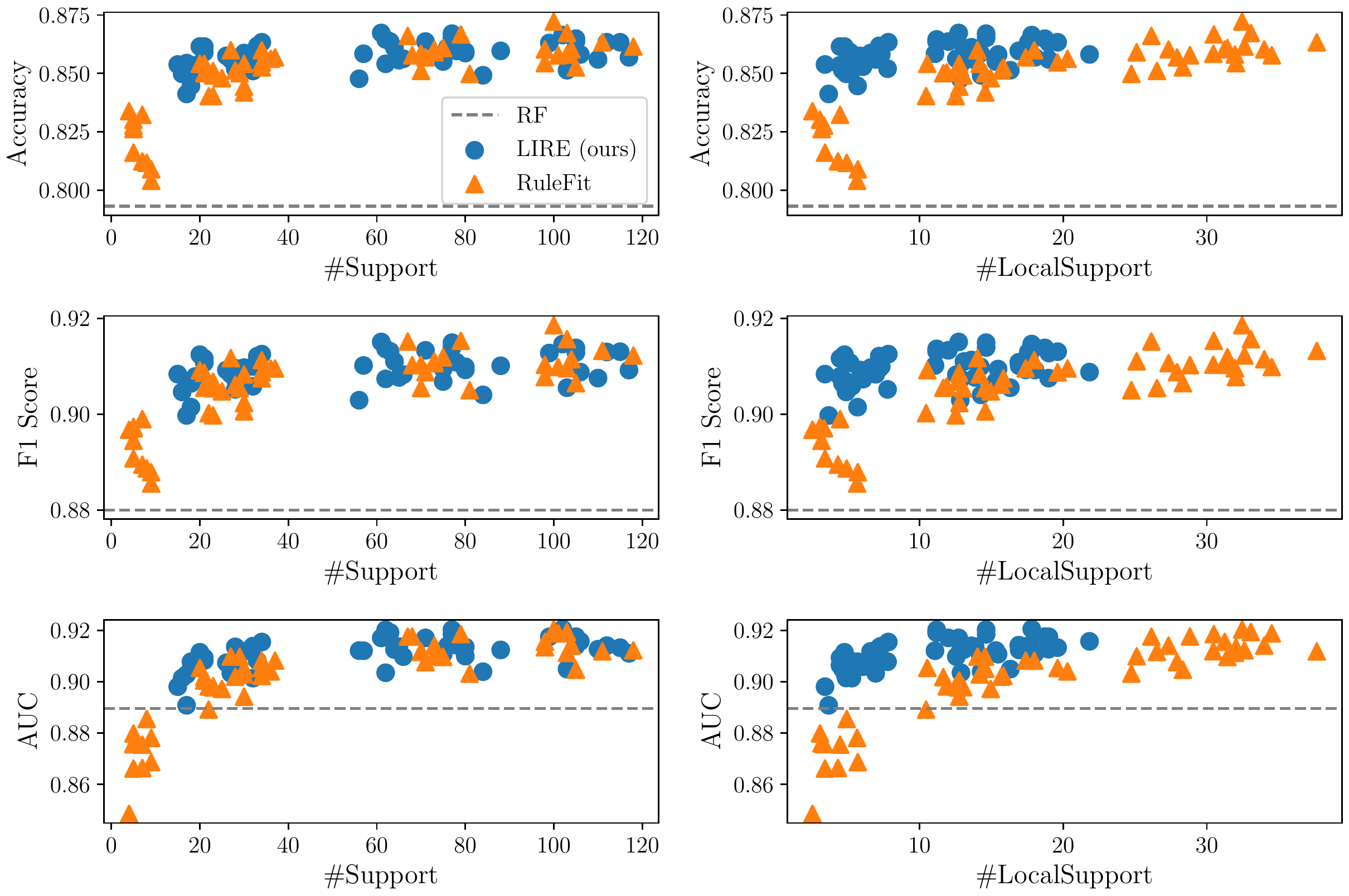}
    \caption{Experimental results on the accuracy-interpretability trade-off with $\lambda = 0.1$. }
    \label{fig:appendx:exp:tradeoff1}
\end{figure}
\begin{figure}[p]
    \centering
    \includegraphics[width=0.95\linewidth]{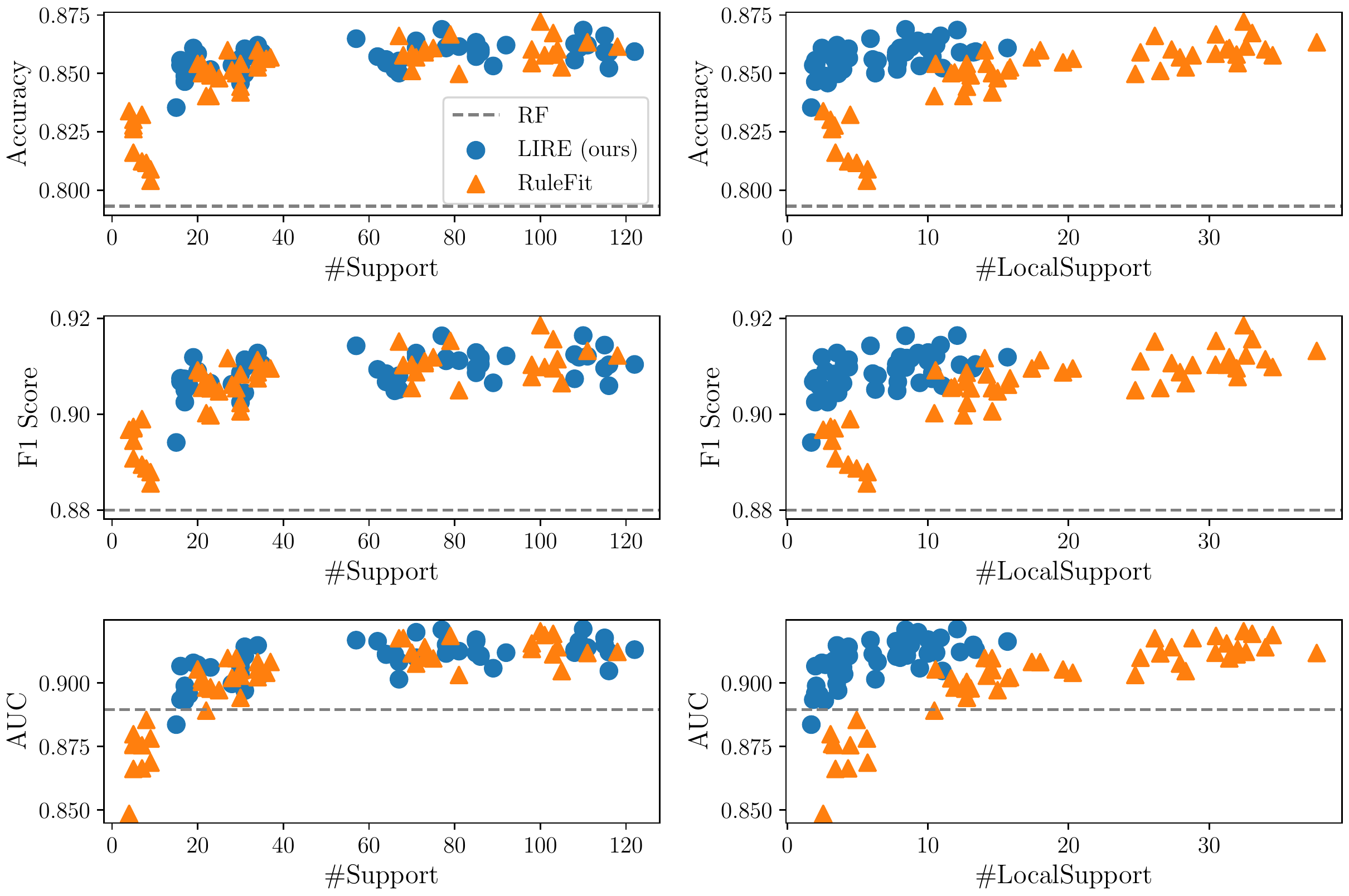}
    \caption{Experimental results on the accuracy-interpretability trade-off with $\lambda = 0.5$. }
    \label{fig:appendx:exp:tradeoff2}
\end{figure}
\begin{figure}[p]
    \centering
    \includegraphics[width=0.95\linewidth]{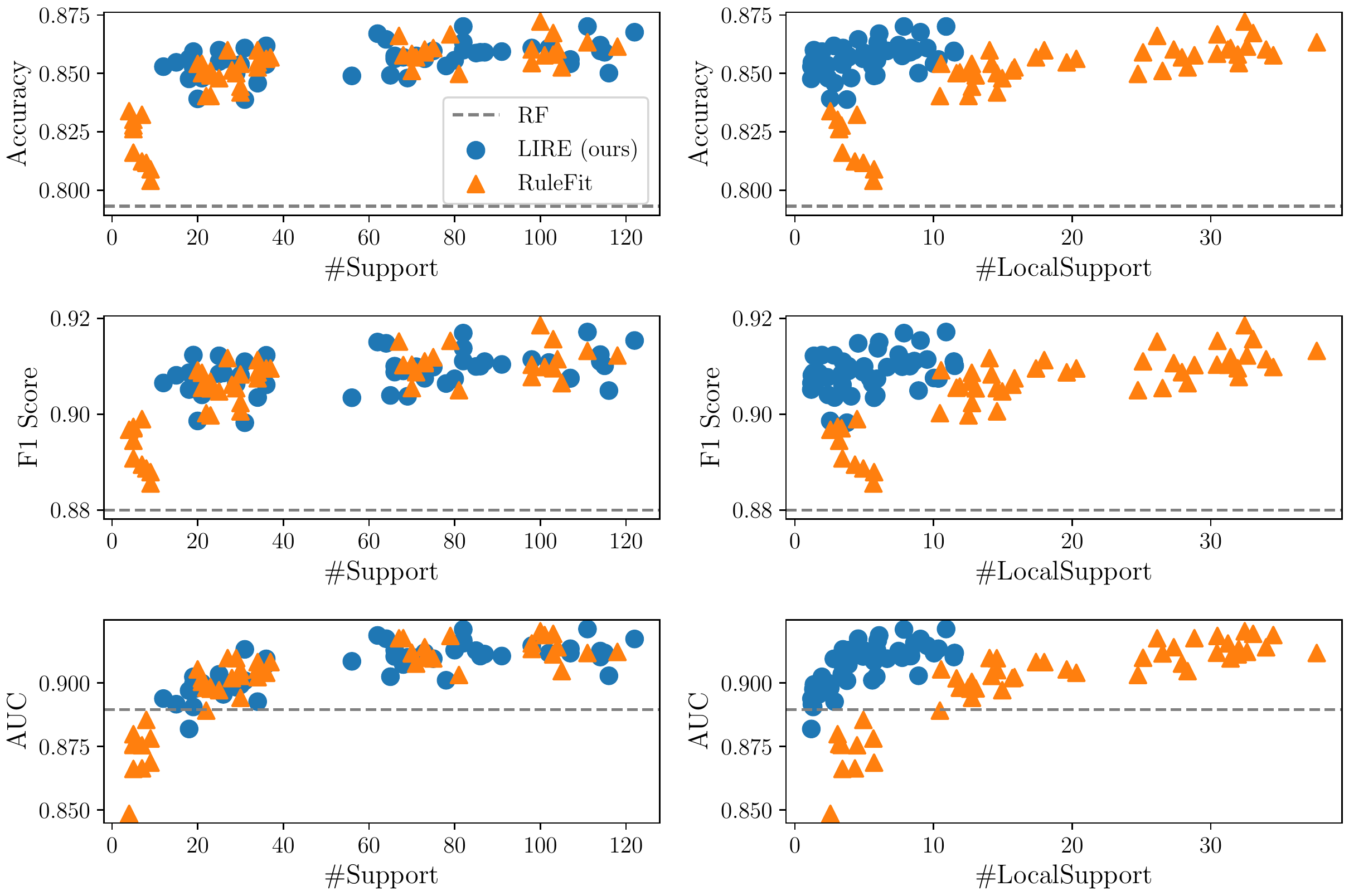}
    \caption{Experimental results on the accuracy-interpretability trade-off with $\lambda = 1.0$. }
    \label{fig:appendx:exp:tradeoff3}
\end{figure}

\begin{figure}[p]
    \centering
    \includegraphics[width=0.9\linewidth]{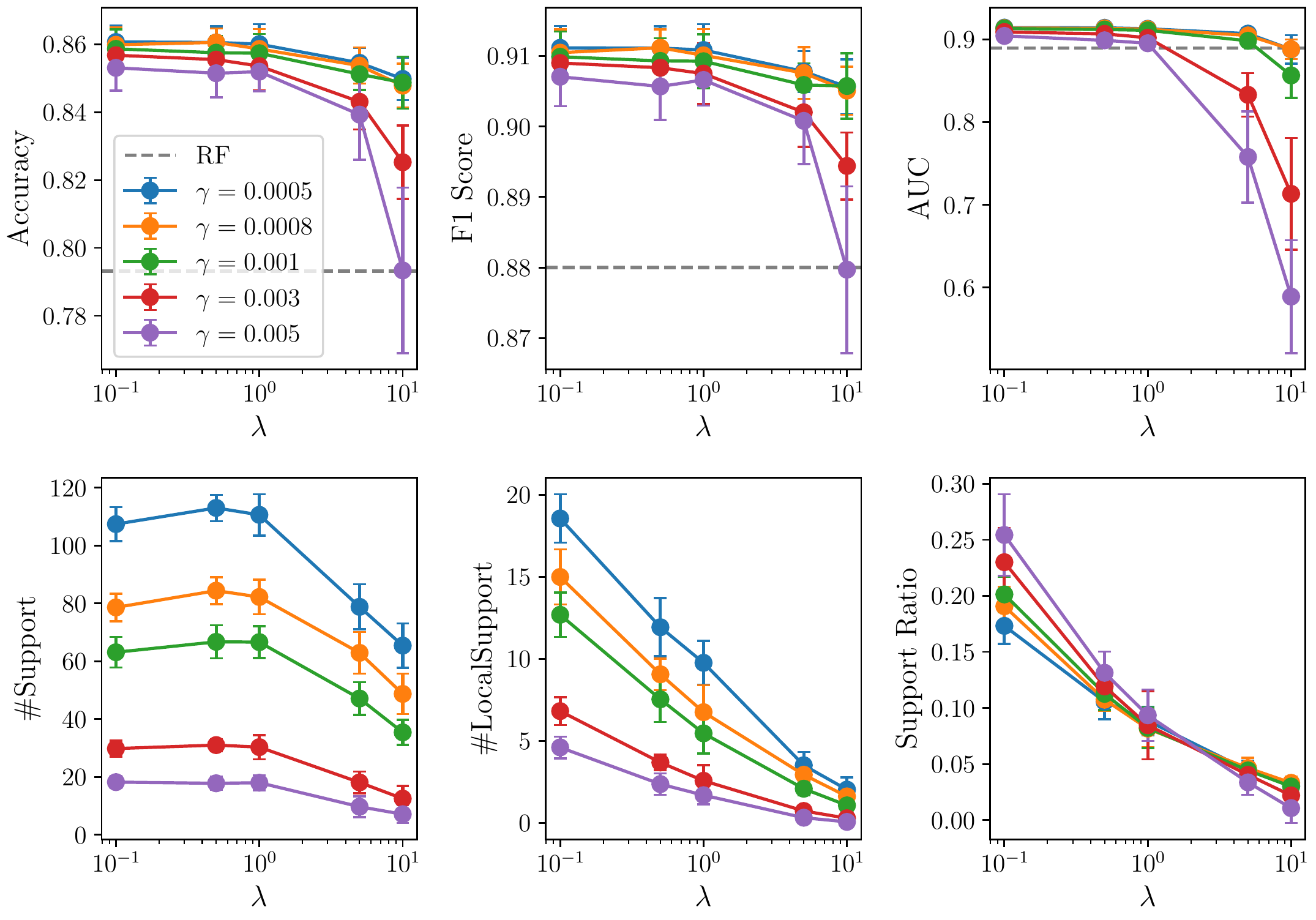}
    \caption{Experimental results on the sensitivity of the trade-off parameter $\lambda$.}
    \label{fig:appendix:exp:sensitivity}
\end{figure}

\begin{figure}[p]
    \centering
    \includegraphics[width=0.9\linewidth]{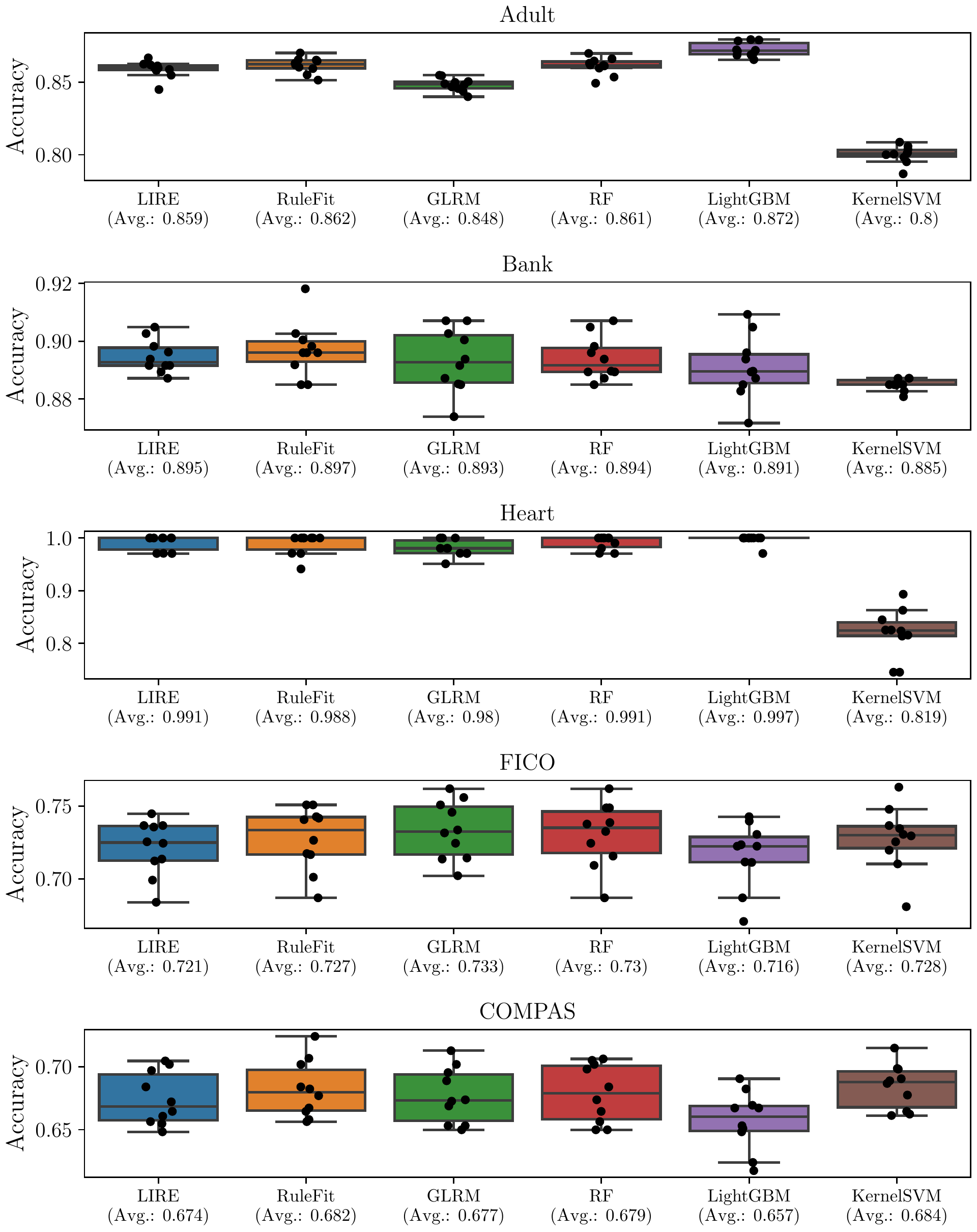}
    \caption{Experimental results on the test accuracy of 10-fold cross-validation. }
    \label{fig:appendix:exp:accuracy}
\end{figure}
\begin{figure}[p]
    \centering
    \includegraphics[width=0.9\linewidth]{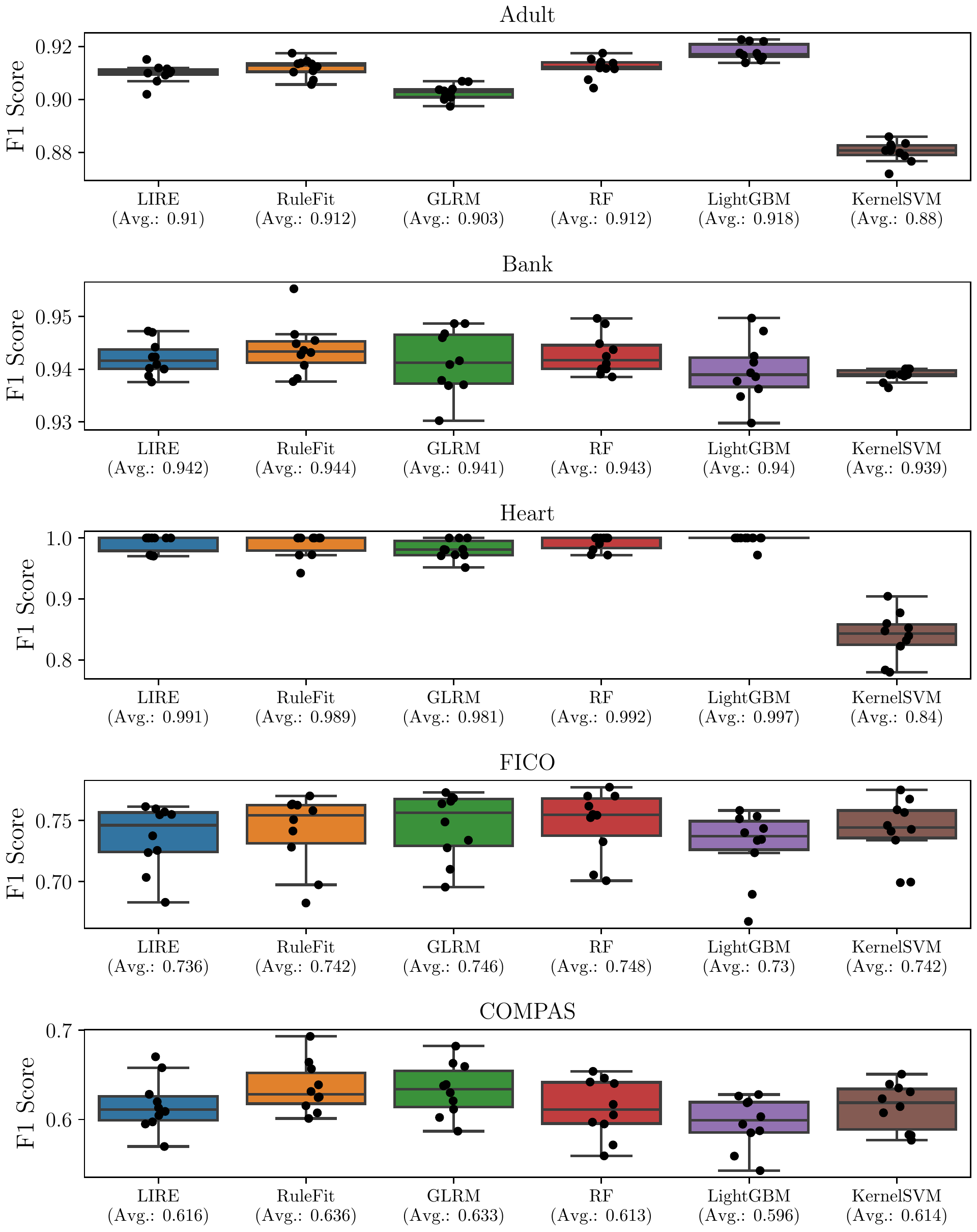}
    \caption{Experimental results on the test F1 score of 10-fold cross-validation. }
    \label{fig:appendix:exp:f1}
\end{figure}
\begin{figure}[t]
    \centering
    \includegraphics[width=0.9\linewidth]{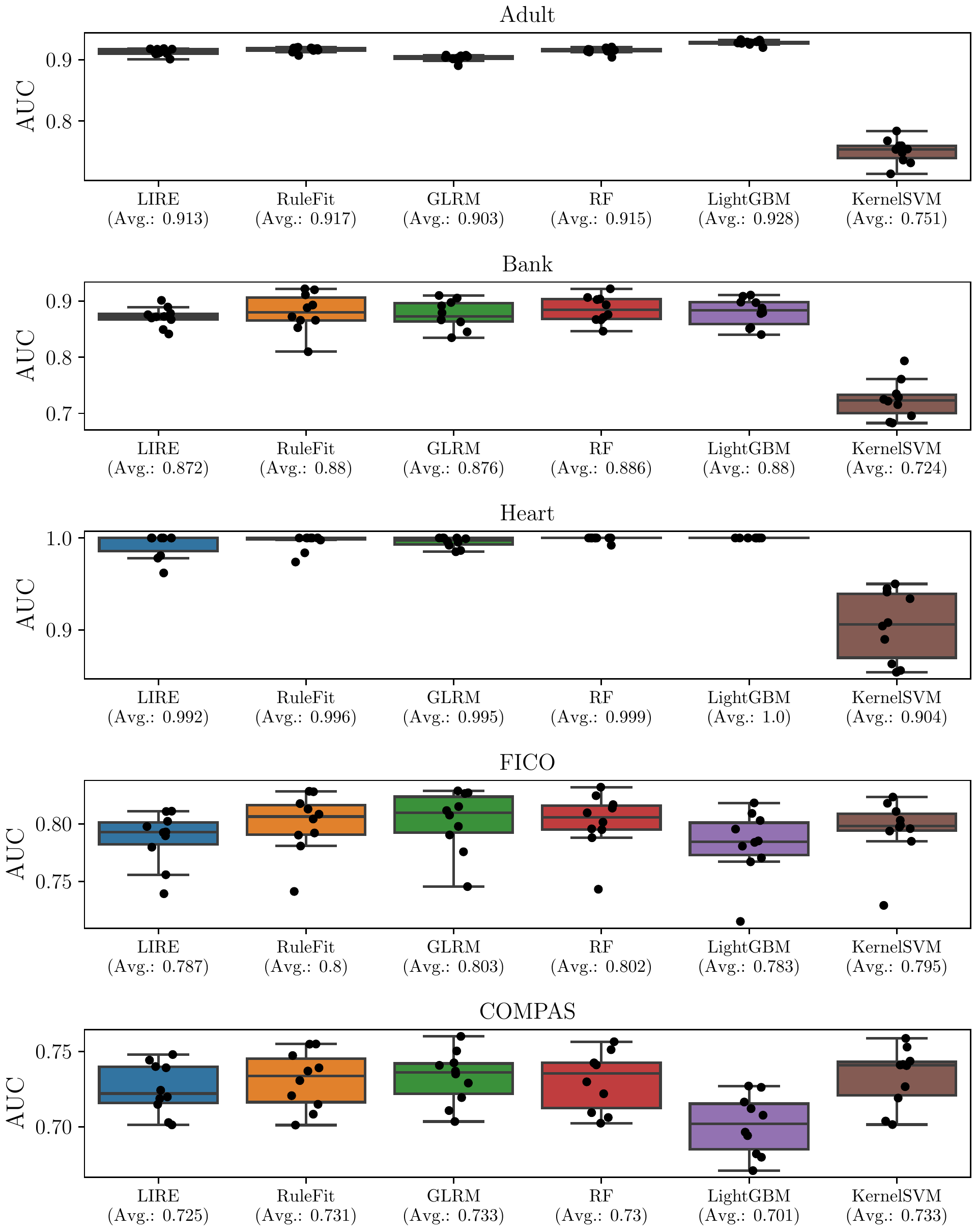}
    \caption{Experimental results on the test AUC of 10-fold cross-validation. }
    \label{fig:appendix:exp:auc}
\end{figure}

\newpage

\begin{table}[p]
    \centering
    % \small
    \caption{
        p-values of the Wilcoxon signed-rank test on the test accuracy, F1 score, and AUC of $10$-fold cross-validation. 
        While the null hypothesis is that there is no difference between the accuracy of LIRE and the baseline, the alternative is that each score of LIRE is lower than the baseline. 
        $\dag$ and $\ddag$ denote the null hypothesis is rejected with significance level $0.05$ and $0.01$, respectively. 
    }
    \begin{tabular}{
        lcccccc
        % p{0.13\textwidth}>{\centering}        
        % p{0.125\textwidth}>{\centering}        
        % p{0.13\textwidth}>{\centering}
        % p{0.13\textwidth}>{\centering}
        % p{0.13\textwidth}>{\centering}
        % p{0.13\textwidth}>{\centering\arraybackslash}
        % p{0.13\textwidth}
    }
        \toprule
        \multirow{2}{*}{\textbf{Criterion}} & \multirow{2}{*}{\textbf{Dataset}} & \multicolumn{2}{c}{\textbf{Rule Ensemble}} & \multicolumn{3}{c}{\textbf{Complex Model}} \\
        \cmidrule(lr){3-4} \cmidrule(lr){5-7}
        & & {RuleFit} & {GLRM} & {RF} & {LightGBM} & {KernelSVM} \\
        \midrule
        \multirow{5}{*}{\textbf{Accuracy}}
        & Adult & $0.02531^\dag$ & $1.0$ & $0.06174$ & $0.00098^\ddag$ & $1.0$ \\
        & Bank & $0.1875$ & $0.53906$ & $0.65234$ & $0.88214$ & $0.99626$ \\
        & Heart & $0.79289$ & $0.91734$ & $0.39323$ & $0.08986$ & $1.0$ \\
        & FICO & $0.02441^\dag$ & $0.00865^\ddag$ & $0.00195^\ddag$ & $0.95801$ & $0.03316^\dag$ \\
        & COMPAS & $0.00888^\ddag$ & $0.1377$ & $0.10267$ & $1.0$ & $0.04292^\dag$ \\
        \midrule
        \multirow{5}{*}{\textbf{F1 score}}
        & Adult & $0.00684^\ddag$ & $1.0$ & $0.01367^\dag$ & $0.00098^\ddag$ & $1.0$ \\
        & Bank & $0.1377$ & $0.46094$ & $0.24609$ & $0.91346$ & $0.99023$ \\
        & Heart & $0.5$ & $0.89796$ & $0.34292$ & $0.08986$ & $1.0$ \\
        & FICO & $0.11621$ & $0.01367^\dag$ & $0.00293^\ddag$ & $0.99023$ & $0.06543$ \\
        & COMPAS & $0.00195^\ddag$ & $0.00488^\ddag$ & $0.8125$ & $0.97559$ & $0.61523$ \\
        \midrule
        \multirow{5}{*}{\textbf{AUC}}
        & Adult & $0.01367^\dag$ & $1.0$ & $0.02441^\dag$ & $0.00098^\ddag$ & $1.0$ \\
        & Bank & $0.1377$ & $0.21582$ & $0.01367^\dag$ & $0.05273$ & $1.0$ \\
        & Heart & $0.2326$ & $0.24948$ & $0.07206$ & $0.0544$ & $1.0$ \\
        & FICO & $0.00195^\ddag$ & $0.00195^\ddag$ & $0.00098^\ddag$ & $0.8125$ & $0.02441^\dag$ \\
        & COMPAS & $0.01367^\dag$ & $0.00488^\ddag$ & $0.01855^\dag$ & $1.0$ & $0.00195^\ddag$ \\
        \bottomrule
    \end{tabular}
    \label{tab:appendix:exp:test:acc}
\end{table}

\begin{table}[p]
    \centering
    % \small
    \caption{
        p-values of the Wilcoxon signed-rank test on the global and local support sizes of $10$-fold cross-validation. 
        While the null hypothesis is that there is no difference between the support size of LIRE and the baseline, the alternative is that the support size of LIRE is larger than the baseline. 
        $\dag$ and $\ddag$ denote the null hypothesis is rejected with significance level $0.05$ and $0.01$, respectively. 
    }
    \begin{tabular}{
        lcccc
        % p{0.125\textwidth}>{\centering}        
        % p{0.13\textwidth}>{\centering}
        % p{0.13\textwidth}>{\centering}
        % p{0.13\textwidth}>{\centering\arraybackslash}
        % p{0.13\textwidth}
    }
        \toprule
        \multirow{2}{*}{\textbf{Dataset}} & \multicolumn{2}{c}{\textbf{\#Support}} & \multicolumn{2}{c}{\textbf{\#LocalSupport}} \\
        \cmidrule(lr){2-3} \cmidrule(lr){4-5}
        & {RuleFit} & {GLRM} & {RuleFit} & {GLRM} \\
        \midrule
        Adult & $1.0$ & $0.09668$ & $1.0$ & $1.0$ \\
        Bank & $1.0$ & $0.83887$ & $1.0$ & $1.0$ \\
        Heart & $1.0$ & $0.99902$ & $1.0$ & $1.0$ \\
        FICO & $0.90332$ & $0.8623$ & $1.0$ & $1.0$ \\
        COMPAS & $0.99618$ & $0.8125$ & $1.0$ & $1.0$ \\
        \bottomrule
    \end{tabular}
    \label{tab:appendix:exp:test:sup}
\end{table}

\end{document}